\documentclass{article}
\usepackage{graphicx} 

\newif\ifshownavigationpage
\newif\ifshowreminders
\newif\ifshownotationindex
\newif\ifshowtheoremlinks
\newif\ifshowtheoremtree
\newif\ifshowtheoremlist
\newif\ifshowequationlist

\newif\ifshowcomments
\newif\ifhighlight 
\newif\ifelaborate

\newif\ifshowaddressedcomments
\newif\ifshowrvin
\newif\ifshowrvout

\usepackage{dsfont}
\usepackage{amsthm, amsmath, amssymb}
\usepackage{thmtools}
\usepackage{soul}
\usepackage{authblk}
\usepackage{framed}
\usepackage{mathrsfs}
\usepackage{hyperref}
\hypersetup{
    colorlinks,
    linkcolor={red!50!black},
    citecolor={blue!50!black},
    urlcolor={blue!80!black}
}

\usepackage[ruled,vlined, linesnumbered, algo2e]{algorithm2e}

\usepackage[dvipsnames]{xcolor}
\usepackage{graphicx, multicol}
\usepackage{marvosym, wasysym}
\usepackage{fancyhdr}
\usepackage{stackengine}
\usepackage{algorithm}
\usepackage{bm}
\usepackage[noend]{algpseudocode}
\usepackage{bbm}
\usepackage{verbatim}
\usepackage{mdframed}
\usepackage{needspace}
\makeatletter
\renewcommand{\ALG@beginalgorithmic}{\scriptsize}
\makeatother
\usepackage{xcolor}
\allowdisplaybreaks

\DeclareFontFamily{U}{mathx}{}
\DeclareFontShape{U}{mathx}{m}{n}{ <-> mathx10 }{}
\DeclareSymbolFont{mathx}{U}{mathx}{m}{n}
\DeclareFontSubstitution{U}{mathx}{m}{n}

\newcommand{\stleq}{\underset{ \text{s.t.} }{\leq}}

\ifhighlight
\else
    \renewcommand{\hl}[1]{#1}
\fi

\newcommand{\rvoutopacity}{20}
\ifshowrvin

\else

\fi

\ifshowrvout
    \newcommand{\rvout}[1]{{\color{red!\rvoutopacity}{#1}} }
    \newcommand{\chrout}[1]{{\color{blue!\rvoutopacity}{#1}} }    
    \newcommand{\rvoutm}[1]{{\color{black!\rvoutopacity}{\ifmmode\text{\sout{\ensuremath{\displaystyle#1}}}\else\sout{#1}\fi}} } 
\else
    \newcommand{\rvout}[1]{}
    \newcommand{\chrout}[1]{}
\fi

\ifshowcomments
    \newcommand{\summ}[1]{{\color{blue}[summary: #1]} } 
    \newcommand{\chr}[1]{{\color{PineGreen}[CR: #1]} } 
    \newcommand{\xw}[1]{{\color{RoyalBlue}[XW: #1]} } 
    \ifshowaddressedcomments
        \newcommand{\chra}[1]{{\color{PineGreen}\sout{[CR: #1]}} } 
        \newcommand{\xwa}[1]{{\color{RoyalBlue}\sout{[XW: #1]}} } 
    \else
        \newcommand{\chra}[1]{} 
        \newcommand{\xwa}[1]{} 
    \fi
\else
    \newcommand{\summ}[1]{} 
    \newcommand{\chr}[1]{} 
    \newcommand{\chra}[1]{} 
    \newcommand{\xw}[1]{} 
    \newcommand{\xwa}[1]{} 
\fi

\usepackage[shortlabels]{enumitem}

\newlist{thmdependence}{itemize}{10}
\setlist[thmdependence]{nosep,label=-}

\ifshowtheoremlinks
    \newcommand{\linksinthm}[1]{\emph{\linkdest{location, #1}\linktopf{#1} \linktothmtree{location, thm tree #1} }}
    \newcommand{\linksinthmwopf}[1]{\emph{\linkdest{location, #1} \linktothmtree{location, thm tree #1} }}
    \newcommand{\linksinpf}[1]{\linkdest{location, proof of #1}\linktothm{#1} \linktothmtree{location, thm tree #1} }
\else
    \newcommand{\linksinthm}[1]{}
    \newcommand{\linksinthmwopf}[1]{}
    \newcommand{\linksinpf}[1]{}
\fi

\ifshownotationindex

\else
    
\fi
    
\newcommand{\linktopf}[1]{\hyperlink{location, proof of #1}{\pflinksymbol}}
\newcommand{\linktothm}[1]{\hyperlink{location, #1}{\thmlinksymbol}}
\newcommand{\linktothmtree}[1]{\hyperlink{#1}{\thmtreelinksymbol}}
\newcommand{\thmlinksymbol}{{\tiny [Theorem]}}
\newcommand{\pflinksymbol}{{\tiny [Proof]}}
\newcommand{\thmtreelinksymbol}{{\tiny [ThmTree]}}

\makeatletter
\newcommand{\linkdest}[1]{%
    \Hy@raisedlink{\raisebox{5mm}[0pt][0pt]{\hypertarget{#1}{}}}%
}
\makeatother

\newcommand{\elaborateopacity}{50}
\newcommand{\elaboratecolor}{RawSienna}
\ifelaborate
    \newcommand{\elaborate}[1]{{\color{\elaboratecolor!\elaborateopacity}{
    \begin{framed}
    \noindent {\footnotesize[Elaboration]}
    #1 
    \end{framed}
    }}\noindent}
\else
    \newcommand{\elaborate}[1]{}
\fi

\usepackage[margin=1.2in]{geometry}

\newtheorem{theorem}{Theorem}
\newtheorem{lemma}[theorem]{Lemma}
\newtheorem*{lemma*}{Lemma}
\newtheorem{corollary}[theorem]{Corollary}
\newtheorem{proposition}[theorem]{Proposition}

\newtheorem{assumption}{Assumption}
\newtheorem{remark}{Remark}

\newtheorem*{theorem-nonumber}{Theorem}
\newtheorem*{condition-nonumber}{Condition}
\newtheorem*{proposition-nonumber}{Proposition}

\usepackage{mathtools}
\DeclarePairedDelimiter{\ceil}{\lceil}{\rceil}
\DeclarePairedDelimiter\floor{\lfloor}{\rfloor}

\newcommand{\cmt}[1]{#1} 
\renewcommand{\cmt}[1]{} 
\renewcommand{\P}{\mathbf{P}}

\newcommand{\E}{\mathbf{E}}
\newcommand{\RV}{\mathcal{RV}}
\newcommand{\R}{\mathbb{R}}

\renewcommand{\complement}{c}

\newcommand{\lo}{\mathit{o}}
\newcommand{\bo}{{O}}

\usepackage[normalem]{ulem}

\usepackage{multirow}
\usepackage{longtable}

\usepackage{array}
\def\delequal{\mathrel{\ensurestackMath{\stackon[1pt]{=}{\scriptscriptstyle\Delta}}}}
\def\distequal{\mathrel{\ensurestackMath{\stackon[1pt]{=}{\scriptstyle d}}}}

\usepackage{mathtools}

\counterwithin{equation}{section}
\counterwithin{lemma}{section}
\counterwithin{corollary}{section}
\counterwithin{theorem}{section}
\counterwithin{definition}{section}
\counterwithin{proposition}{section}
\counterwithin{figure}{section}
\counterwithin{table}{section}

\algrenewcommand\algorithmicrequire{\textbf{Require:}}
\algrenewcommand\algorithmicensure{\textbf{Postcondition:}}

\title{Multi-agent Multi-armed Bandit with Fully Heavy-tailed Dynamics}
\author{Xingyu Wang 
        \\ \href{mailto:x.wang4@uva.nl}{x.wang4@uva.nl} 
        \\ Quantitative Economics, University of Amsterdam 
        \\ Amsterdam, 1018 WB, NL

   \and Mengfan Xu\footnote{Corresponding Author} \\  \href{mailto:mengfanxu@umass.edu}{mengfanxu@umass.edu} \\ 
   Mechanical and Industrial Engineering,  University of Massachusetts Amherst\\
   Amherst, MA, 01003, USA}
\begin{document}

\maketitle

\begin{abstract}
\noindent
We study decentralized multi-agent multi-armed bandits 
in fully heavy-tailed settings, where clients communicate over sparse random graphs with heavy-tailed degree distributions and observe heavy-tailed (homogeneous or heterogeneous) reward distributions with potentially infinite variance. 
The objective is to maximize system performance by pulling the globally optimal arm with the highest global reward mean across all clients. 
We are the first to address such fully heavy-tailed scenarios, which capture the dynamics and challenges in communication and inference among multiple clients in real-world systems.
In homogeneous settings,
our algorithmic framework exploits hub-like structures unique to heavy-tailed graphs,
allowing clients to aggregate rewards and reduce noises via hub estimators when constructing UCB indices;
under $M$ clients and degree distributions with power-law index $\alpha > 1$, 
our algorithm attains a regret bound (almost) of order $O(M^{1 -\frac{1}{\alpha}} \log{T})$.
Under heterogeneous rewards,
clients synchronize by communicating with neighbors, aggregating exchanged estimators in UCB indices;
With our newly established information delay bounds on sparse random graphs, 
we prove a regret bound of $O(M \log{T})$.
Our results improve upon existing work, which only address time-invariant connected graphs, or 
light-tailed dynamics in dense graphs and rewards.
\end{abstract}

\section{Introduction}


Multi-armed Bandit (MAB) is an online sequential decision-making framework where a decision maker, or a client, pulls an arm from a finite set of arms at each time step, receives the reward of the pulled arm, and aims to maximize the cumulative received reward, or equivalently, minimize the regret compared to always pulling the optimal arm \cite{auer2002finite, auer2002nonstochastic}. 
Recent advancements have focused on its multi-agent variant, namely Multi-agent Multi-armed Bandit (MA-MAB), capturing the complexity of networks of decision makers in real-world scenarios.
A notable focus in on cooperative MA-MAB, where $M$ clients aim to optimize the regret of the entire network via communication, with respect to a globally optimal arm defined by the global reward averaged across all clients. 
In this work,
we consider a general and widely applicable setting with decentralization, where clients communicate on a graph without the stringent assumption of a central server.


Unique to decentralized MA-MAB is the coupling between graphs and time in the sequential regime, 
where two clients can communicate only when there is an edge between them on the graph.
Many studies have focused on time-invariant graphs (i.e., with edges remaining constant over time).
However, the emergence of examples such as ad-hoc wireless networks \cite{roman2013features} has motivated the recent interest in time-varying graphs. 
In particular, the random-graph setting allows graphs to be redrawn from a distribution at each time step.
For instance, \cite{xu2023decentralized} has recently explored 
the classical light-tailed Erdos-Renyi graphs in MA-MAB, where the edge between each pair of clients is sampled under a Bernoulli distribution with constant probability.
Notably, this model is symmetric and dense, with each client having exactly the same expected degree (i.e., the count of connected clients) of order $\bo(M)$.
However, real-world tasks often involve uneven and sparse communications:
the degree---representing the communication resource assigned to each client---can be rather heterogeneous among clients, and the resource for each individual may not scale with $M$. 
Enabling collaboration over such asymmetric and sparse graphs facilitates system-wide performance in cooperative MA-MAB, which depends on the collective performance of all clients rather than any single individual,
and has profound impact on promoting fairness and enhancing social good. 
This setting, however, remains unexplored and leaves a significant research gap.


This work focuses on sparse random graphs with power-law heavy-tailed degree distributions that capture the highly asymmetric network dynamics in real-world tasks, where a few vertices play hub-like roles with many connections to others, while most vertices have limited degrees. 
\rvout{More surprisingly, they lead to sparse graphs, where the expected degree of each vertex is uniformly bounded even as the graph size $M$ tends to infinity.}%
Such power-law heavy tails prevail in
real-world networks across numerous contexts, including finance and economics, transportation networks, online retailing, supply chains, social communications, and epidemiology, to name a few;
see, e.g., \cite{DACRUZ2013189,clancy2021epidemicscriticalrandomgraphs,hearnshaw2013complex,10.1145/2487788.2488173,doi:10.1137/070710111,PhysRevE.65.066130,PhysRevE.65.035108}.
Therefore, it is essential to address this research gap and build
solid theoretical and algorithmic foundations for MA-MAB over sparse and heavy-tailed random graphs,
thus enabling efficient and robust decision-making via time-evolving, uneven, and limited communication and coordination over a wide range of real-world networks.

Two aspects of the reward distributions are central to multi-agent MAB problems. 
First, if the expected reward of an arm is the same for all clients, it is categorized into homogeneous settings; otherwise, it is heterogeneous. 
While homogeneous MA-MAB is generally well understood, heterogeneous MA-MAB has recently gained attention and presents additional challenges.
This is particularly true in the decentralized setting, as inferring the global optimal arm requires the reward information from all clients. 
While Erdos-Renyi graphs have been addressed in \cite{xu2023decentralized},
their approach does not apply to our sparse and asymmetric setting. 
Second, 
the intensity of randomness (more specifically, tail behaviors) of reward distributions significantly impacts
the complexity of the problem. 
Aside from the long-standing focus on sub-Gaussian rewards in the bandit community,
there has been a recent interest in distributions with heavier tails, including the sub-Exponential class and Exponential families \cite{NIPS2013_aba3b6fd,JIA2021728}, and those with even heavier tails and infinite $p^\text{th}$ order moments \cite{bubeck2013bandits, 6516952,dubey2019thompson,pmlr-v151-tao22a}.
Inferring the mean value using reward observations under such extreme randomness becomes quite challenging, often necessitating the use of robust estimators in the algorithms.
Notably, such efforts are mostly limited to single-agent MAB. 
A recent work \cite{dubey2020cooperative} addresses heavy-tailed rewards with $(1+\epsilon)^\text{th}$-order moments in MA-MAB, but only considers the homogeneous-rewards setting and with time-invariant connected graphs.  
Considering heavy-tailed rewards in a more challenging and general setting with heterogeneous rewards and time-varying random graphs remains unexplored, a gap we address herein as well

In this paper, we focus on the following question: \emph{Can we formulate and solve the multi-agent multi-armed bandit problem with heavy-tailed random graphs and heavy-tailed rewards in both homogeneous and heterogeneous settings?}

\subsection{Main Contributions}

We hereby provide an affirmative answer to the research question through our contributions, elaborated as follows. We formulate the multi-agent multi-armed bandit problem over sparse, asymmetric, and heavy-tailed random graphs, and under rewards with potentially infinite variance.
Specifically, we consider rank-1 inhomogeneous random graphs \cite{boguna2003class,doi:10.1073/pnas.252631999} with heavy-tailed degree distributions,
a standard setting in literature 
(e.g., \cite{PhysRevE.95.022307,van_der_Hofstad_van_der_Hoorn_Litvak_Stegehuis_2020}).
In this framework,
the probability for having an edge between a pair of clients at each time step is dictated by (normalization of) attraction weights of clients;
under heavy-tailed weight distributions, some clients consistently play hub-like roles and often connect to many other clients.
Moreover, the graphs we consider are much more sparse (with $\bo (1)$ expected degree for each client) 
compared to Erdos-Renyi graphs (with $\bo (M)$ expected degree),
which translates to significantly reduced communication costs and is of broad interest in large-scale multi-agent learning problems.

Methodologically, we propose new algorithms for homogeneous and heterogeneous settings. 
In the homogeneous-reward setting,
we characterize and exploit the notion of hubs exclusive to heavy-tailed graphs:
clients over the hub communicate and aggregate rewards, achieving variance reduction proportional to the hub size, while other clients use delayed aggregation through a hub representative. 
This principle guides the design of our novel UCB index, which also incorporates the median-of-means estimator for robust estimation under heavy-tailed rewards.
In the heterogeneous setting, another challenge is asynchronization (differences in arm pulls) among clients due to variations in their reward distributions. 
To address this, clients use random sampling when asynchronization occurs, and deploy UCB-based strategies otherwise based on our newly constructed reward estimators.
Specifically, we propose an aggregation method that integrates the most recent heavy-tailed reward information from all clients, introducing novel information update mechanisms.

We establish theoretical guarantees for the proposed algorithms through comprehensive regret analyses. In homogeneous settings, we obtain a regret upper bound that is (almost) of order $O(M^{1-\frac{1}{\alpha}}\log T)$, which is sublinear in $M$. This improves upon the potential linearity in \cite{dubey2020cooperative} and demonstrates sample complexity reduction even under sparse graph structures. Under heterogeneous rewards, we derive an upper bound of order $O(M\log{T})$, 
extending the bound in \cite{xu2024decentralized} to sparse and asymmetric graphs with heavy-tailed rewards. The results highlight the consistency, effectiveness, and robustness of our approach.

The paper is organized as follows. 
Section~\ref{sec: problem formulation} sets notations and formulates the problem. 
Section~\ref{sec:graphs} explores properties that are exclusive to sparse, heavy-tailed graphs and useful for our MA-MAB setting. 
Then, we propose algorithms and conduct regret analyses for the homogeneous and heterogeneous settings in Sections \ref{sec:homo} and \ref{sec:heter}, respectively. Lastly, we conclude the paper and discuss future work in Section \ref{sec:conclusion}. 

\section{Problem Formulation}\label{sec: problem formulation}


We start with notations used throughout this paper.
Given a positive integer $k$, let ${[k]} = \{1,2,\ldots,k\}$.
We adopt the convention that $[0] = \emptyset$.
Given two sequences of non-negative real numbers $(x_n)_{n \geq 1}$ and $(y_n)_{n \geq 1}$, 
we say that $x_n = \bo (y_n)$ (as $n \to \infty$) if there exists some $C \in [0,\infty)$ such that $x_n \leq C y_n\ \forall n\geq 1$.
Besides, we say that $x_n = \lo (y_n)$ if $\lim_{n \rightarrow \infty} x_n/y_n = 0$.


Let $M$ denote the number of clients, which are labeled by $[M] = \{1,2,\ldots,M\}$. At each time $t = 1,2,\ldots$, the clients are distributed over an undirected graph $G_t = (V,E_t)$, where $V = [M]$, and $E_t$ is the set of edges (i.e. two clients communicate at time $t$ only if $(i,j) \in E_t$) generated by the following distribution:
independently for each pair, $(i,j) \in E_t\text{ with probability }P(h_i,h_j),
    \text{ and }(i,j) \notin E_t\text{ with probability }1 - P(h_i,h_j)$;
    where 
\begin{align}
            P(u,v) = \min\{ 1,\ uv /(\theta M) \}
            \qquad\forall u,v \geq 0,
            \label{def: kernel for the random graph}
\end{align}
$(h_i)_{i \geq 1}$ are independent copies of a positive random variable $h$,
and $\theta = \E h$.
Note that this model is the standard rank-1 inhomogeneous random graphs; e.g., 
\cite{boguna2003class,doi:10.1073/pnas.252631999}.
Intuitively speaking, $h_i$ is the weight assigned to the $i^\text{th}$ node, representing its \emph{attraction} to the other nodes.
The weight $h_i$ does not change with time $t$, and is close to the expected degree of the $i^\text{th}$ node (especially under large $M$) over the graphs $G_t$'s.
We use $\mathcal{N}_i(t)$ to denote the neighborhood set of the $i^\text{th}$ node at time $t$, which includes all the other nodes that are connected to $i$ over the graph $G_t$. 
Equivalently, the graph $G_t$ can be represented by the adjacency matrix $(X_{i,j}^t)_{1 \leq i,j \leq M}$
where $X_{i,j}^t = 1$ if there is an edge between nodes $i$ and $j$, and $X_{i,j}^t = 0$ otherwise. 
We set $X^t_{i,i} \equiv 1$ for any $1 \leq i \leq M$.
We also define the empirical adjacency matrix by
$P_t(i,j) \delequal \frac{\sum_{s=1}^tX_{i,j}^s}{t}$ and $P_t = \big(P_t(i,j)\big)_{1 \leq i,j \leq M}$. We note that
 each node $m$ only knows its own neighbors and can only observe the $m$-th row of $P_t$, i.e., node $m$ has access to $\{P_t(m,j)\}_j$ but not $\{P_t(k,j)\}_j$ for $k \neq m$.

Power-law heavy tails are typically captured through the notion of regular variation.
Given a measurable function $\phi:(0,\infty) \to (0,\infty)$, we say that $\phi$ is regularly varying as $x \rightarrow\infty$ with index $\beta$ (denoted as $\phi(x) \in {\RV_\beta}(x)$ as $x \to \infty$) 
if $\phi(x) = x^\beta \cdot l(x)$ for some function $l:(0,\infty) \to (0,\infty)$
with
$\lim_{x \rightarrow \infty}l(tx)/l(x) = 1$ for all $t>0$.
That is, $\phi(x)$ roughly follows a power-law tail with index $\beta$.
For a standard treatment on the properties of regularly varying functions, see, e.g., \cite{resnick2007heavy}.
The next assumption specifies the law of weights $h_i$ and the choice of $\theta$ in \eqref{def: kernel for the random graph}.

\begin{assumption}[Heavy-Tailed Graph]
\label{assumption: heavy-tailed graph}
The sequence $(h_i)_{i \geq 1}$ are iid copies of $h$, with law $\P( h > x) \in \RV_{-\alpha}(x)$ for some $\alpha > 1$,
and $\theta = \E h$ in \eqref{def: kernel for the random graph}.
\end{assumption}

We impose the next assumption to exclude the pathological case that some clients are (almost) never connected to others at any time $t$.

\begin{assumption}[Lower Bound for $h$]
\label{assumption: lower bound for h}
There exists $c_h > 0$ such that $\P(h \geq c_h) = 1$.
\end{assumption}

We add two comments for Assumptions~\ref{assumption: heavy-tailed graph} and \ref{assumption: lower bound for h}:
(1) Assumption~\ref{assumption: lower bound for h} does not prevent isolated clients in $G_t$;
in fact, given time $t \geq 1$ and client $i$, the probability that $i$ is not connected to any others over $G_t$ equals
$
\prod_{ j \in [M]:\ j \neq i }\big[ 1 - \E P(h_i,h_j)  \big],
$
which is strictly positive;
(2) the graph $G_t$ is both \emph{heavy-tailed} and \emph{sparse};
in particular, the heavy-tailedness in $h$ implies that that the client with the highest weight $h_i$ can have degree of order $M^{\alpha}$;
however, since the law of $h$ does not vary with $M$, the expectation of the degree for each client is only $\bo (1)$ (i.e., most nodes are only connected to a small number of nodes).


Let $K$ be the number of arms. For each client $i \in [M]$, we denote the reward of arm $1 \leq k \leq K$ at time $t$ by $r_k^i(t)$, which is an i.i.d.\ sequence from a time-invariant distribution with mean value $\mu_k^i$. When $\mu^{i}_k = \mu^{j}_k\ \forall i,j \in [M]$ holds for any arm $k$, it is referred to as a homogeneous-reward setting; otherwise it is heterogeneous. 
Our Assumption~\ref{assumption: heavy-tailed reward} is sufficiently general to account for 
both heavy-tailed reward distributions (potentially with infinite variance) and light-tailed distributions (with finite moments of any order, e.g.\ sub-Gaussian and sub-exponential classes).
\begin{assumption}[Rewards with Uniformly Bounded $(1+\epsilon)^\text{th}$ Central Moments]
\label{assumption: heavy-tailed reward}
Given $i \in [M]$ and $k \in [K]$, rewards $(r^{k}_i(t))_{t \geq 1}$ are iid copies from the distribution $F^{k}_i$.
Furthermore, there exist $\epsilon \in (0,1]$ and $\rho \in (0,\infty)$ such that $\sup_{ i \in [M],\ k \in [K] }\E
         | r^k_i(1) - \mu^k_i |^{1+\epsilon}
    \leq \rho,$
where we use $\mu^k_i \delequal \E r^k_i(1) =  \int x F^{k}_i(dx)$ to denote expected rewards.
\end{assumption}

We use $a_m^t$ to denote the arm pulled by client $m$ at time $t$. We define the global reward of arm $i$ at 
each time step $t$ as $r_i(t) = \frac{1}{M}\sum_{m=1}^Mr_i^m(t)$,
and the expected value of the global reward of arm $i$ by
$\mu_i = \frac{1}{M}\sum_{m=1}^M\mu_i^m$.
We denote the \emph{global optimal arm} by
$i^* = \arg\max_i\mu_i$,
and consider the cooperative setting where
all clients would, ideally, pull the globally optimal arm $i^*$.
The optimality gap for arm $i$ is
$
\Delta_i = \mu_{i^*} - \mu_i.
$
This motivates the definition of the global regret by $R_T = T \cdot \mu_{i^*} - \frac{1}{M}\sum_{t=1}^T\sum_{m=1}^M\mu_{a_m^t},$
which measures the difference in the cumulative expected reward between the global optimal arm and the action sequence. 
The main objective of this paper is to develop a multi-agent MAB algorithm and minimize $R_T$ for clients given the sparse communications available on $(G_t)_{t \geq 1}$.

\section{Analyses on Random Graphs}\label{sec:graphs}

In this section, we establish useful properties for the hub structures and information delay over random graphs $G_t$ introduced in Section~\ref{sec: problem formulation}.
The results lay the foundation for our subsequent analysis of multi-agent multi-armed bandits.
We collect the proofs in the Appendix.

\subsection{Hub on Heavy-Tailed Graphs}\label{sec:graphs-sub1}

A feature exclusive to heavy-tailed graphs is the arise of hub-like nodes with disproportionally large degrees (i.e., being connected to a large number of nodes), which enables efficient communication among all clients through hubs.
We first consider a \emph{deterministic} characterization of the hub.
Let $\hat i$ be the client with the highest degree at time $1$ (arbitrarily pick one if there are ties).
Let $S^t_0 \delequal \{ i \in [M]: (i,\hat i) \in E_t \}$ be the clients communicating with $\hat i$ at time $t$.
Note that, for any $i \in [M]$ with $h_i > \theta M / h_{\hat i}$,
by \eqref{def: kernel for the random graph} we know that such $i$ must be \emph{deterministically} (i.e., with probability 1) connected to $\hat i$ for all $t$.
 Lemma~\ref{lemma, hub: deterministic lower bound} confirms that, with high probability, such nodes $i$ are plenty.
Specifically, by standard techniques in extreme value theory for heavy-tailed variables, 
one can show that $h_{\hat i}$ is roughly of order $M^{1/\alpha}$,
and (under $\alpha \in (1,2)$) the count of $i \in [M]$ with $h_i > \theta M / h_{\hat i} \approx \theta M^{1 - \frac{1}{\alpha}}$  is roughly of order $M^{2 - \alpha}$.
By taking the $\zeta$-slackness in Lemma~\ref{lemma, hub: deterministic lower bound}, we are able to ensure the exponentially decaying bound for pathological cases.

\begin{lemma}\label{lemma, hub: deterministic lower bound}
Let Assumptions~\ref{assumption: heavy-tailed graph} and \ref{assumption: lower bound for h} hold with $\alpha \in (1,2)$.
Given $\zeta \in (0,2 - \alpha)$, there exists $\gamma > 0$ such that 
\begin{align*}
    \P( |S_0| \leq M^{ 2 - \alpha - \zeta})
    =
    \lo(\exp( -M^\gamma)),
\end{align*}
where $S_0 = \cap_{t \geq 1}S^t_0$.
\end{lemma}

In fact, we can further improve upon Lemma~\ref{lemma, hub: deterministic lower bound} by considering the following \emph{stochastic} characterization of hubs.
Given $\zeta > 0$, let 
$
\tau(t) \delequal \max\{u \leq t:\ |S^u_0| > M^{ \frac{1}{\alpha} - \zeta}\}
$
be the last time the hub is large (w.r.t.\ threshold $M^{ \frac{1}{\alpha} - \zeta }$) up until time $t$, under the convention that $\tau(t) = 0$ when taking maximum over empty sets.
Lemma~\ref{lemma, hub: stochastic lower bound} bounds the time gap between the emergence of large hubs.
The proof builds upon extreme value theory and a straightforward bound of $\sup_{t \leq T}t - \tau(t)$ using geometric random variables.

\begin{lemma}\label{lemma, hub: stochastic lower bound}
Let Assumptions~\ref{assumption: heavy-tailed graph} and \ref{assumption: lower bound for h}  hold.
Define event
$
A_{\alpha,\zeta}
    \delequal
    \{
        h_{\hat i} \geq M^{ \frac{1}{\alpha} - \frac{\zeta}{2} }
    \}. 
$
Let $\zeta \in (0,1 - \frac{1}{\alpha})$.
There exists $\gamma > 0$ such that 
$$
\P\big( (A_{\alpha,\zeta})^\complement\big) = 
\lo \big( \exp(-M^\gamma ) \big).
$$
Furthermore, 
there exists $M_0 > 0$ such that for any $M \geq M_0$ and $T \geq 1$, 
$$\P\bigg( \sup_{t \leq T}t - \tau(t) > \log T\ \bigg|\ A_{\alpha,\zeta}\bigg) 
    \leq 
    \frac{1}{MT}.$$
\end{lemma}

\paragraph{Comparison to dense and light-tailed graphs}
\cite{xu2023decentralized} considered dense E-R graphs,
where any pair of clients connects on a regular basis.
In this work, we show that
under the presence of heavy tails in degree distributions,
clients can afford to collaborate over much sparse communication by sending messages to and receiving messages from the hub center $\hat i$ that integrates all information.

\subsection{Information Delay over Sparse Graphs}\label{sec:graphs-sub2}

The sparsity of graphs $G_t$ makes existing analysis in \cite{xu2023decentralized} largely incompatible,
and requires new approach to obtain detailed bounds regarding the information delay under sparse communication.
Specifically, given some non-empty subset of clients $S \subseteq [M]$,
let $\bar S^0 = S$, and
$
\bar S^t \delequal 
\{
i \in [M]:\ 
i \in \bar S^{t - 1};\text{ or }\exists j \in \bar S^{t - 1}\text{ s.t. }(i,j) \in E_t
\}
$
for each $t \geq 1$.
That is, if all clients in $S$ send a piece of message at time $1$, which will passed to neighbors over graph $G_t$ at each time $t$,
then $\bar S^t$ is the collection of clients that have received the message at time $t$.
Lemma~\ref{lemma, information delay over sparse graphs} shows that, with high probability, 
the information delay uniformly for any client $i \in [M]$ is at most $\bo \big((\log M)^2\big)$.
Our proof strategy is to establish a coupling between the sequence of graphs $(G_t)_{t \geq 1}$ and a branching process, whose size grows geometrically fast in expectation.

\begin{lemma}\label{lemma, information delay over sparse graphs}
Under Assumption~\ref{assumption: lower bound for h},
there exists $\kappa \in (0,\infty)$ such that 
$$\P\big( j \notin \bar{S}^{ \gamma \cdot \kappa \cdot (\log M)^2 }\text{ for some }j \in [M]\big)
    \leq M^{-\gamma}$$
    holds
for any $\gamma > 0$, $M \geq 1$, and any non-emtpy $S \subseteq [M]$.
\end{lemma}

\section{Homogeneous Rewards}\label{sec:homo}

This section considers the homogeneous-reward setting. The algorithmic framework will be extended to the heterogeneous-reward setting in the next section.
Specifically, we propose the algorithm in Section \ref{sec:alg-homo}, addressing the challenges from both heavy-tailed rewards and sparse graphs. Then, we establish the theoretical effectiveness of the proposed algorithm in Section \ref{sec:reg-homo}.

\subsection{Algorithm}\label{sec:alg-homo}

Under homogeneous rewards, we propose a new algorithm called HT-HMUCB (\textbf{H}eavy-\textbf{T}ailed \textbf{H}o\textbf{M}ogeneous \textbf{U}pper \textbf{C}onfidence \textbf{B}ounds); pseudocode is provided below. The algorithm consists of several stages in the following order.

 \paragraph{Hub identification.}
 A novel step in our algorithm concerns the identification of hub center $\hat i$, i.e., the client with the highest degree at time $t = 1$.
 Specifically, the degree information of all clients at time $1$ will be passed over $G_t$ at each step $t$ so that, with high probability,
 each client $i \in [M]$ is able to tell whether itself is the hub center (i.e., $i = \hat i$) or not after $\bo\big( (\log M)^2 \big)$ steps;
 {see Algorithm~\ref{alg:find hub center} (Appendix).
 Afterwards, all non-center clients will follow the reward information processed by the hub center $\hat i$.


\paragraph{Arm selection.} During this stage, the clients decide which arm to pull by executing a UCB-based strategy, where each arm $i$ is assigned a UCB index, formally expressed as $\hat{\mu}_i^m(t) + \rho^{ \frac{1}{1+\epsilon} }(\frac{c \log(t)}{N_{m,i}(t)})^{ \frac{\epsilon}{1+\epsilon} }$. Here, $\hat{\mu}_i^m(t)$ and $N_{m,i}(t)$ represent the global reward estimators and sample counts of arm $i$ by client $m$, respectively, defined in Rule 1 below.
Constants $\rho$ and $\epsilon$ are characterized in Assumption~\ref{assumption: heavy-tailed reward}, and $c$ is specified below in Theorem~\ref{thm:1a2}.

\paragraph{Transmission.} 
We define an information filtration $\mathcal{F}_m(t)$ as the information available to $m$ up to time $t$, which reads as $\mathcal{F}_m(0) = \{(m,1), r_i^m(1), N_{m,i}(1), \hat{\mu}_i^m(1)\}, \mathcal{F}_m(t) = \cup_{j \in \mathcal{N}_m(t-1)}\mathcal{F}_j(t-1)$. Each client $m$ communicates with its neighbors $\mathcal{N}_{m}(t)$ by sending an message, composed of $(m,t), r_i^m(t)$, $N_{m,i}(t)$, $\hat{\mu}_i^m(t)$, and $\mathcal{F}_m(t)$, while collecting messages from its neighbors.

\paragraph{Information update.} After pulling arms and receiving feedback from the environment, as well as information from others, the clients proceed to update their information based on \textbf{Rule 1}, which is detailed below:
\\ 
1) Local estimation
\begin{align*}
& t_{m,j} = \max\{ s \leq t:\  (j,s) \in \mathcal{F}_{m}(t)\} \notag \\
  & \text{local sample counts: } n_{m,i}(t+1) = n_{m,i}(t) + \mathds{1}_{a_m^t = i}, 
  \end{align*}
2) Global estimation 
        \begin{align*}
& \text{Sample counts: }  
 N_{m,i}(t+1)  = \textstyle \sum_{m \in \mathcal{N}_{m}(t)}n_{m,i}(t+1) \notag \\
& \text{Estimator: if $m$ is center (i.e., $m = \hat i$),} \\
& \quad \hat{\mu}^{\hat i}_i(t+1) = {MoM_B(\{r_i^{j}(s): r_i^{j}(s) \in \mathcal{F}_m(t)\})} \notag \\
& \text{Estimator: if $m$ is not center, } \hat{\mu}^m_i(t+1) =  \hat{\mu}^{\hat i}_i(\max_{j \in S_0}t_{m,j})
  \end{align*}
\vspace{-5mm}

Here, $MoM_B\big( (X_i)_{i \in [n]} \big)$ denote the \textbf{median of mean} estimator with $B$ batches, i.e., the meadian of the estimators $\hat \mu_1,\ldots, \hat \mu_B$ defined by 
$
\hat \mu_j =  \frac{1}{N}\sum_{ t = (j-1)N + 1  }^{jN}
    X_t
$
with $N = \floor{n/B}$.
MoM estimators have been applied in \cite{bubeck2013bandits} for UBC algorithms in single-agent settings.
Our work further demonstrates its use for robust estimation under heavy tails in multi-agent MAB problems.

\paragraph{Comparison with prior works}  Our algorithm differs from the existing algorithms in \cite{dubey2020cooperative} for homogeneous settings with heavy-tailed rewards in the following ways: 1) the clients identify the hub center to maximize information efficiency, which is computationally more tractable than the clique search in \cite{dubey2020cooperative}; 
and 2) the clients rely solely on information from hub center, rather than maintaining global estimators individually, to reduce noise in estimation.

\begin{algorithm}[h]
\SetAlgoLined
\caption{HT-HMUCB (Heavy-tailed Homogeneous UCB)}\label{alg:dr}
 Initialization: For each client $m$ and arm $i \in \{1,2,\ldots, K\}$, we set  $N_{m,i}(L+1) = n_{m,i}(L)$; all other values at $L+1$ are initialized as $0$\;\par
 \For{
   $t = 1,2,\ldots, L$
 }{
    Indentify hub center $\hat{i}
    $ by running Algo.~\ref{alg:find hub center};\tcp*[f]{Hub} \par
 }
 \For{$t = L + 1 , L + 2, \ldots,T$}{
\For(\tcp*[f]{UCB}){each client m}{
$a_m^t = \arg\max_{i}\hat{\mu}_i^m(t) + \rho^{ \frac{1}{1+\epsilon} }(\frac{c \log(t) }{N_{m,i}(t)})^{ \frac{\epsilon}{1 + \epsilon} }$ \par
Pull arm $a_m^t$ and receive reward $r^{m}_{a_m^t}(t)$\;}
    The environment generates the graph $G_t$;\tcp*[f]{Env} \par
    
    Each client $m$ sends $(m,t)$, $r_i^m(t)$, $N_{j,i}(t)$, $\Tilde{\mu}_i^m(t)$, $\mathcal{F}_m(t)$ to each client in $\mathcal{N}_m(t)$\;
    \tcp*[f]{Transmission}\par
   \For{each client m}{
   \For{ $i =1, \ldots, K$}{
  Update $\bar{\mu}_i^m(t), n_{m,i}(t), N_{m,i}(t)$ and $\Tilde{\mu}^m_i(t)$ based on Rule 1\;\tcp*[f]{Update}
    }}
    }
\end{algorithm}

\subsection{Regret Analyses}\label{sec:reg-homo}

In this section, we demonstrate the effectiveness of the proposed algorithm through regret analyses. 
Notably, 
the tail index $\alpha$ for degree distributions in Assumption~\ref{assumption: heavy-tailed graph}
plays a key role in our regret bound.
First, we establish Theorem~\ref{thm:1a2} for $\alpha \in (1,2)$.

\begin{theorem}[$\alpha \in (1,2)$]\label{thm:1a2}
Let Assumptions~\ref{assumption: heavy-tailed graph}--\ref{assumption: heavy-tailed reward} hold with
$1 < \alpha  <2$.
Let Algorithm~\ref{alg:dr} run under Rule 1.
 Then, 
 given $\zeta \in (0,2-\alpha)$,
 there exists  $\eta >0$ such that, for any $T$ and $M$,
 the event $A_{\zeta,\delta}$ holds with probability at least $( 1 - \frac{2\eta}{M} - \frac{\eta}{TM})$, we have
\begin{align*}
    & \E[R_T | A_{\zeta,\delta}]
     \leq 
    L + 
     M
        \sum_{i}
    \Bigg(
        \frac{2c\Delta_i\log{T}}{ M^{ 2 - \alpha - \zeta } \cdot (\frac{\Delta_i}{2C\rho^{\frac{1}{1+\epsilon}}})^{\frac{1+\epsilon}{\epsilon}}} + \frac{\pi^2}{3} {\Delta_i}
    \Bigg) 
    \\ 
    & = \bo \bigg( (1 + 2M^{\alpha - 1 + \zeta}) \cdot \rho^{\frac{1}{\epsilon}} \sum_{i \in [K]}\Delta_i^{-\epsilon} \cdot \log{T}\bigg),
\end{align*}
where $L = 2\kappa (\log{M})^2\log{T}$, 
$\kappa$ is the constant characterized in Lemma~\ref{lemma, information delay over sparse graphs}, $c = (16 \log{2e^{1/8}})^{\frac{\epsilon}{1 + \epsilon}}$, $C = (12)^{ \frac{1}{1 + \epsilon} }$, we set $B = 8 \log( e^{1/8}T )$ for the count of batches in MoM estimators,
$|S_0|$ denotes the size of $S_0$ (see Lemma~\ref{lemma, hub: deterministic lower bound}), 
and the event is defined by $A_{\zeta, \delta} = A_{\zeta, \delta}^1 \cap A_{\zeta, \delta}^2\cap A_{\zeta, \delta}^3$ with 
$A_{\zeta, \delta}^1 = \{|S_0| \geq M^{ 2 - \alpha - \zeta}\}$, 
$A^2_{\zeta, \delta} = \{n_{m,i}(t_{m,j}) \geq  n_{m,i}(t) - \kappa \log{M}\log{T}\ \forall t \leq T,\ \forall m,i,j  \}$,
and
$
A^3_{\zeta,\delta} = \{\hat i(m) \neq \hat i\text{ for some }m \in [M]\}
$
(see Algo.~\ref{alg:find hub center}).
In particular, $
    \E[R_T | A_{\zeta, \delta}] 
    \leq  O(M^{\alpha - 1+\zeta} \cdot \log{T})  = o(M) \cdot O(\log{T}).$
\end{theorem}


\begin{proof}[Proof Sketch]
The proof hinges on the key observation that, with high probability, $|S_0|$ is at least $M^{2 - \alpha - \zeta}$;
see Lemma~\ref{lemma, hub: deterministic lower bound}.
The communication delay between the clients is then bounded by Lemma~\ref{lemma, information delay over sparse graphs} (see also Lemma~\ref{lemma:information_delay}), 
hence the clients in the hub enjoy the reduction in sample complexity. 
This is achieved by utilizing a concentration inequality with respect to $\sum_{m \in S_0}n_{m,i}(t)$ instead of $n_{m,i}(t)$, resulting in an individual regret of order $\frac{1}{|S_0|} \cdot \log{T}$. Consequently, the total regret is of order $\frac{M - |S_0|}{|S_0|} \cdot \log{T}$.
The full proof is provided in the Appendix. 
\end{proof}

\begin{remark}[Expected regret]\label{rmk:exp-1}
  We emphasize that the expected regret $\E[R_T]$ is upper bounded by 
  $\E[R_T] = \E[R_T|A_{\zeta,\delta}]\P(A_{\zeta,\delta}) + \E[R_T|(A_{\zeta,\delta})^\complement]\P((A_{\zeta,\delta})^\complement)  \leq O(M^{\alpha - 1+\zeta} \cdot \log{T}) + O(M\log{T}) \cdot (\frac{2\eta}{M} + \frac{\eta}{TM} ) =  O(M^{\alpha - 1+\zeta} \cdot \log{T})$. 
  Here, the inequality follows from Theorem~\ref{thm:1a2}, as well as the observation that,
  in the pathological case where the ``good'' event $A_{\zeta,\delta}$ does not happen, the regret would still be no worse than the total regret of multiple single-agent bandits, which results in a regret of order $O(M \log{T})$.
\end{remark}

We stress that the regret bound in Theorem~\ref{thm:1a2} depends on $\alpha$ (see Assumption~\ref{assumption: heavy-tailed graph}), 
which characterizes the relationship between the regret and the heavy-tailed graph dynamics, and reflects the reduction of complexity that is unique to our setting. 
Additionally, the regret bound depends on $\rho$ and $\epsilon$ (see Assumption~\ref{assumption: heavy-tailed reward}), 
capturing the influence of the heavy-tailed rewards. 
The dependency on optimality gaps $\Delta_i$'s is standard in MAB, 
but in our case there is an additional $\epsilon$-polynomial factor due to heavy-tailed rewards.

In fact, we can obtain even stronger regret bound---and even without the requirement of $\alpha \in (1,2)$ in Theorem~\ref{thm:1a2}---by considering the the stochastic characterization of hubs $S^t_0$.
This is demonstrated in Theorem~\ref{thm:1a}.

\begin{theorem}[$\alpha  > 1 $]\label{thm:1a}
Let Assumptions~\ref{assumption: heavy-tailed graph}--\ref{assumption: heavy-tailed reward} hold.
Let Algorithm~\ref{alg:dr} run under Rule 1.
Then, 
given $\zeta \in (0,2-\alpha)$
there exists $\eta > 0$ such that, for any $T$ and $M$, 
the event $A_{\alpha, \delta, \zeta}$ holds with probability at least
$(1 - \frac{\eta}{M} - \frac{\eta}{TM})$, 
and
\begin{align*}
    \E[ R_T | A_{\alpha, \delta, \zeta}]
    & 
    \leq  L +
    \textstyle M\sum_{i}
    (\frac{2c\log{T}}{M^{\frac{1}{\alpha} - \zeta}(\frac{\Delta_i}{2C\rho^{\frac{1}{1+\epsilon}}})^{\frac{1+\epsilon}{\epsilon}}} + \frac{\pi^2}{3}
    ) \Delta_i
    \\ 
    & = \bo \Big(M^{ 1 - \frac{1}{\alpha} + \zeta } \cdot \log{T}\Big),
\end{align*}
where $A_{\alpha, \delta, \zeta} = A_{\zeta,\delta} \cap A_{\alpha, \zeta}$,
with event $A_{\alpha,\zeta}$ defined in Lemma~\ref{lemma, hub: stochastic lower bound}
and event $A_{\zeta,\delta}$ defined in Theorem~\ref{thm:1a2},
and parameters $L,\kappa,c,C,B$ are specified as in Theorem~\ref{thm:1a2}.
\end{theorem}

\begin{proof}[Proof Sketch]
The information delay characterized in Lemma~\ref{lemma, information delay over sparse graphs} (and hence Lemma~\ref{lemma:information_delay}) for sparse graphs still holds here.
Meanwhile,
Lemma~\ref{lemma, hub: stochastic lower bound} proves that hub size $|S^t_0|$, despite being time-varying, is often times large (of order $M^{ \frac{1}{\alpha} - \frac{\zeta}{2}  }$).
This tighter lower bound for hub size
leads to further variance reduction:
as clients can efficiently collect information sent by (often times) even larger $S_0^t$, 
 and thereby minimize regret. 
The detailed proof is in Appendix.
\end{proof}


\begin{remark}[Expected regret]\label{rmk:exp-2}
Analogous to Remark \ref{rmk:exp-1}, we can derive an upper bound on $\E[R_T]$, which has the same order as the high-probability bound above, i.e., $\E[R_T] \leq O(M^{1- \frac{1}{\alpha } + \zeta} 
\log{T})$. 
\end{remark}

\begin{remark}[Comparison to Theorem~\ref{thm:1a2}]
    Given $\alpha > 1$,  the index $1 - \frac{1}{\alpha }$ for regret bound in Theorem~\ref{thm:1a}
    is always smaller than the index $2 - \alpha$ in Theorem~\ref{thm:1a2}, 
    due to the preliminary inequality $-\frac{1}{x} \leq 1 - x$ for any $x > 1$.
    This implies that by considering a dynamic, time-dependent hub, the coefficient in our regret bound is upper bounded by a smaller constant compared to Theorem \ref{thm:1a2}. 
    This highlights the advantage of leveraging the tighter characterization for notion of time-varying hub as in Lemma~\ref{lemma, hub: stochastic lower bound} and the significant contributions of this work,
    as
    we obtain a regret upper bound of even smaller order and under more relaxed assumptions (i.e., without the requirement of $\alpha < 2$).
\end{remark}


\paragraph{Comparison with existing literature on homogeneous MA-MAB}

We begin the comparison by focusing on the perspective of regret bound and emphasizing the order of $M$, as naive UCB already leads to a regret of order $\log{T}$. 
First, the regret bound in \cite{dubey2020cooperative} is $R_T \leq O(\alpha(G)\rho^{\frac{1}{\epsilon}}(\sum_{i}\Delta_i^{-\epsilon}) \log{T})$, 
and their algorithm relies on solving an NP-hard problem to find the clique (i.e., the largest independent set). 
Also, the quantity  $\alpha(G)$---the independence number of graph $G$---may not have an explicit form and has been an active research topic.
For a connected graph, some known bounds on $\alpha(G)$ are \cite{willis2011bounds}: $\alpha(G) \leq M - \frac{M-1}{\Delta}$, where $\Delta$ is the maximum degree, and $\alpha(G) \geq \frac{M}{1 + \Delta}$, implying that $R_T \leq O(M \cdot \log{T})$. 
In contrast, we obtain an improved regret bound that is sub-linear in $M$.
Specifically, since the slackness parameter $\zeta$ in Theorem~\ref{thm:1a} can be set arbitrarily close to $0$,
our regret bound is almost of order 
 $O(M^{1-\frac{1}{\alpha}+\zeta})$,
which improves upon the regret bound $O(M\log{T})$ of executing individual bandits without communication in homogeneous MAB (i.e., single-player bandit). 
In particular, our bound is established for sparse graphs with total degree (i.e., count of edges over the graph) of order $\bo (M)$.
While \cite{yangcooperative} establishes a regret bound of order $O(\log{T})$ independent of $M$, 
the authors assume that the graph is connected or complete, which can have $\bo(M^2)$ total degree,
and they only consider sub-Gaussian rewards. 
Their assumptions allow the use of arm elimination instead of UCB, which is largely different from our problem setting.
    
Regarding assumptions,
we 1) do not require the graph to be connected (or $l$-periodically connected, as in \cite{zhu2023distributed}); 2) allow the graph to change over time; 
and 3) address sparse graphs. 
Points 1) and 3) address limitations present in almost all existing work on multi-agent multi-armed bandit problems, to the best of our knowledge, while point 2) resolves an open problem identified in \cite{dubey2020cooperative}, which suggested time-varying network analysis and tested this case numerically. 
Our heavy-tailed reward assumption is in the same spirit as that in \cite{dubey2020cooperative}.
To the best of our knowledge, our framework, the most general to date in homogeneous MA-MAB, relaxes these widely used assumptions and thus opens new research directions.

\section{Heterogeneous Rewards}\label{sec:heter}

This section addresses the more general heterogeneous setting.
Specifically, we present the algorithm in Section~\ref{sec:alg-heter} and the regret analysis in Section \ref{sec:reg-heter}.
The results are well beyond the scope of the existing work
 on MA-MAB with random graphs and heterogeneous rewards \cite{xu2023decentralized}, 
 which only considered limited to Erdős–Rényi graphs with light-tailed, dense dynamics.

\subsection{Algorithm}\label{sec:alg-heter}

\begin{algorithm}[h]
\SetAlgoLined
\caption{HT-HTUCB (Heavy-Tailed Heterogeneous UCB): Learning period}\label{alg:heter-L}
 Initialization: For each client $m$ and arm $i \in \{1,2,\ldots, K\}$, we have $\Tilde{\mu}^m_i(L+1)$, $N_{m,i}(L+1) = n_{m,i}(L)$; all other values at $L+1$ are initialized as $0$\;\par
 \For{$t = L + 1 , L + 2, \ldots,T$}{
\For(\tcp*[f]{UCB}){each client m}{
\eIf{there is no arm $i$ such that $n_{m,i}(t) \leq N_{m,i}(t) - 2\kappa (\log{M})^2 \log{T}$}{ $a_m^t = \arg\max_{i}\hat{\mu}_i^m(t) + \rho^{ \frac{1}{1+\epsilon} }(\frac{c \log(t) }{N_{m,i}(t)})^{ \frac{\epsilon}{1 + \epsilon} }$}
{ Randomly sample an arm $a_m^t =  t \mod K$
}
Pull arm $a_m^t$ and receive reward $r^m_{a_m^t}(t)$\;
}
    The environment generates the graph $G_t = (V,E_t)$;\tcp*[f]{Env} \par
    Each client $m$ sends $(m,t)$, $r_i^m(t)$, $N_{j,i}(t)$, $\bar{\mu}_i^m(t)$, $\Tilde{\mu}_i^m(t)$, $\mathcal{F}_m(t)$ to $\mathcal{N}_m(t)$\; \tcp*[f]{Transmission}\\
   \For{each client m}{
   \For{ $i =1, \ldots, K$}{
  Update $n_{m,i}(t), N_{m,i}(t)$ and $\Tilde{\mu}^m_i(t)$ based on Rule 2\;
    }}
    }
\end{algorithm}

Under heterogeneous rewards, we propose a new algorithm, namely HT-HTUCB (\textbf{H}eavy-\textbf{T}ailed \textbf{H}e\textbf{T}erogeneous UCB), 
as the presence of heterogeneity in rewards necessitates 
different approaches to informative communication and information updates across clients. 
The pseudo-code is provided in Algorithm \ref{alg:heter} (Appendix) and Algorithm \ref{alg:heter-L} for the burn-in and learning stages, respectively.

Algorithm \ref{alg:heter} (Appendix) for the burn-in period is designed to accumulate local reward and graph information,
and is identical to that in \cite{xu2023decentralized} so we defer the details to Appendix. 
This prepares the clients with initial reward and graph estimators, enabling them to communicate and integrate information in the subsequent learning stage. Specifically, during the burn-in period, clients pull each arm $1 \leq i \leq K$ sequentially and update the average reward of each arm as local reward estimators. Simultaneously, clients observe the graph, updating the edge frequency and average degree of clients. At the end of the burn-in period, the clients output an initial global estimator, calculated as the weighted average of the local reward estimators, using the edge frequencies as weights.

Moving onto the learning period, clients employ UCB-based strategies to pull arms, communicate with each other, and integrate the information collected from neighbors. 
The steps are executed in the following order.

\paragraph{Arm Selection.} 
We still employ a UCB-based strategy, but the global estimator is constructed differently, 
with an additional condition during the execution of the UCB-based strategy: $n_{m,i}(t) \leq N_{m,i}(t) - 2\kappa (\log{M})^2\log{T}$. 
This novel condition ensures that clients remain synchronized and accounts for the longer information delay caused by heavy-tailed graph dynamics, which differs from the setting in \cite{xu2023decentralized, zhu2023distributed}. 

\paragraph{Transmission.} 
This step is almost identical to that of Section \ref{sec:alg-homo}, except for the message components. Here, an information filtration $\mathcal{F}_m(t)$ reads as $\mathcal{F}_m(0) = \{(m,1), r_i^m(1), N_{m,i}(1), \bar{\mu}_i^m(1), \tilde{\mu}_i^m(1)\}$ and $\mathcal{F}_m(t) = \cup_{j \in \mathcal{N}_m(t-1)}\mathcal{F}_j(t-1)$. Each client $m$ communicates with its neighbors $\mathcal{N}_{m}(t)$ by sending an message, including $(m,t), r_i^m(t)$, $N_{m,i}(t)$, $\bar{\mu}_i^m(t)$, $\Tilde{\mu}_i^m(t)$ and $\mathcal{F}_m(t)$, while collecting messages from its neighbors. 

\paragraph{Information update.} 
Since the heterogeneity in rewards necessitates obtaining reward information from all clients,
we propose a new information update step to aggregate information, represented by \textbf{Rule 2}:
\\
1) Local estimation
\begin{align*}
  & \text{local sample counts: } n_{m,i}(t+1) = n_{m,i}(t) + \mathds{1}_{a_m^t = i}, \notag \\
  & \text{Local estimator: } \bar{\mu}^m_i(t+1) = {MoM_B(\{r_i^{m}(s)\}_{1 \leq s \leq t})}  
  \end{align*}
2) Global estimation 
        \begin{align*}
& t_{m,j} = \max\{ s \leq t:\ (j,s) \in \mathcal{F}_{m}(t)\} \notag \\
  & \hat{N}_{m,i}(t+1) = \max \{n_{m,i}(t+1), \{N_{j,i}(t)\}_{j \in \mathcal{N}_{m}(t)}\} \notag \\
  & \Tilde{\mu}^m_i(t+1) = \textstyle \sum_{j=1}^M P^{\prime}_t\Tilde{\mu}^m_{i,j}(t_{m,j}) + d_{m,t}\textstyle \sum_{j }\hat{\mu}^m_{i,j}(t_{m,j})  \notag \\
  & P_{t}^{\prime} = \frac{(N - M2^{\frac{1}{(\epsilon + 1)}})}{MN2^{\frac{1}{\epsilon + 1}}}, N = (12^{ \frac{1}{1 + \epsilon} })^{\frac{1+\epsilon}{\epsilon}} + 1 \notag \\
  & d_{m,t} = \frac{(1- \sum_{j=1}^M P^{\prime}_t)}{M} \notag 
\end{align*}

\paragraph{Comparison with existing work and Section \ref{sec:alg-homo}}

Compared to Section \ref{sec:alg-homo}, the arm selection step imposes an extra requirement that $n_{m,i}(t) \leq N_{m,i}(t) - 2\kappa(\log{M})^2\log{T}$ in UCB,
which balances exploitation and exploration given the noise in the global reward estimator $\hat{\mu}$ in the heterogeneous case. 
Secondly, we remove the hub estimation step 
in the heterogeneous setting, 
because each client must collect information from all clients by message passing with through neighbors  sets $\mathcal{N}_m(t)$.
Lastly, we run Rule 2 instead of Rule 1 for information update.  

Compared to \cite{xu2023decentralized}, 
our UCB index
$\hat{\mu}_i^m(t) + \rho^{ \frac{1}{1+\epsilon} }(\frac{c \log(t) }{N_{m,i}(t)})^{ \frac{\epsilon}{1 + \epsilon} }$
address (potentially) heavy-tailed rewards,
while \cite{xu2023decentralized} considers $\hat{\mu}_i^m(t) + F(m,i,t)$ where $ F(m,i,t) = \sqrt{\frac{C_1\ln{t}}{n_{m,i}(t)}}$ (for sub-Gaussian rewards) and $F(m,i,t) = \sqrt{\frac{C_1\ln{T}}{n_{m,i}(t)}} + \frac{C_2\ln{T}}{n_{m,i}(t)}$ (for sub-exponential rewards). 
Also, we propose new the information update rule due to the differences in reward and graph dynamics.

\subsection{Regret Analysis}\label{sec:reg-heter}

Importantly, we next demonstrate the effectiveness of Algorithm \ref{alg:heter} and Algorithm \ref{alg:heter-L} by examining the regret upper bound. 
In particular, we stress that Theorem~\ref{thm:heter} does not rely on Assumption~\ref{assumption: heavy-tailed graph},
meaning that the results address 
both light-tailed and heavy-tailed degree distributions in the random graph model
\eqref{def: kernel for the random graph}.

\begin{theorem}\label{thm:heter}
Let Assumptions~\ref{assumption: lower bound for h} and \ref{assumption: heavy-tailed reward} hold.
Let Algorithm \ref{alg:heter-L} run with Rule 2.
Then, we have that with $\P(A_{\zeta, \delta}) \geq 1 - 7/T$ and
\begin{align*}
    \E[R_T|A_{\zeta, \delta}]  
    & \leq  2\kappa K(\log{M})^2 \log{T}
    + 
    \sum_{i \in [K]}M\Delta_i \cdot 
    \Bigg(
        \max\bigg\{\frac{2cN\log{T}}{(\frac{\Delta_i}{2\rho^{\frac{1}{1+\epsilon}}})^{\frac{1+\epsilon}{\epsilon}}}, 2\kappa \log{M}\log{T}\bigg\}
    \Bigg) 
    \\ 
    &\quad  +   \sum_{i \in [K]}M\Delta_i \cdot \bigg(\frac{2\pi^2}{3} + 2\kappa (\log{M})^2 \log{T}\bigg)  
    \\
    & = \bo (M\log{T})
\end{align*}
where 
$K$ is the count of arms,
the event is defined by
$A_{\zeta, \delta} = \{n_{m,i}(t_{m,j}) \geq  n_{m,i}(t) - \kappa (\log{M})^2\log{T}\ \forall t \leq T,\ \forall m,i,j  \}$, $N = (12^{ \frac{1}{1 + \epsilon} })^{\frac{1+\epsilon}{\epsilon}} + 1$, 
and
$c, B,\kappa, \rho, \epsilon$ are defined in Theorem \ref{thm:1a2}. 
\end{theorem}

\begin{proof}[Proof Sketch]
Note that we do not exploit the hub structure in this setting, as clients need to collect information from all other clients rather than relying solely on a hub that contains only a subset of information.
Nevertheless, 
the information delay bounds in Lemma~\ref{lemma, information delay over sparse graphs} (and hence Lemma~\ref{lemma:information_delay}) for sparse graphs still hold.
 Using the estimators constructed in Rule 2, 
 which leverage neighbor information to collect and integrate global information, we prove a concentration inequality for the global estimator $\hat{\mu}$ with respect to the global mean values: 
 $$| \hat{\mu}_m^i(n_{m,i}(t);k)  - \mu_m^i | \leq 2\rho^{ \frac{1}{1+\epsilon} }(\frac{Mc \log(T) }{\min_{m}n_{m,i}(t)})^{ \frac{\epsilon}{1 + \epsilon} }.$$ 
 This ensures that clients can identify the globally optimal arm using UCB after
 $${2cM\log{T}}\bigg/{\bigg(\frac{\Delta_i}{2\rho^{\frac{1}{1+\epsilon}}}\bigg)^{\frac{1+\epsilon}{\epsilon}}}$$ steps with high probability. Another possible scenario is that, when clients are not synchronized, 
 they randomly select arms instead of using UCB, which is proved to be upper bounded by the second constant term. 
 As a result, the total number of pulls of sub-optimal arms $n_{m,i}(t)$ can be bounded by $O(\log{T})$. 
 Lastly, regret decomposition shows that the regret is dominated by ${n_{m,i}(t)}$,
 and thus has the upper bound.
See Appendix for the detailed proof.
\end{proof}

\begin{remark}[Extension]
We emphasize that this theorem holds for both light-tailed  and heavy-tailed random graphs
as it does not rely on Assumption~\ref{assumption: heavy-tailed graph}.
This also highlights a key difference between the homogeneous and heterogeneous settings: in the homogeneous case, heavy-tailed $h_i$'s lead to sample complexity reduction compared to the light-tailed case (with regret of order $O(M\log{T})$), whereas the heterogeneous result holds true for both light-tailed and heavy-tailed distributions for the attraction weight $h_i$'s. 
\end{remark}

\paragraph{Comparison with the existing work} We analyze the regret order with respect to $T$, as the absence of communication can lead to regret of order $\bo(T)$ \cite{xu2023}. 
The most relevant work \cite{xu2023decentralized} establishes a regret upper bound of $O(\log{T})$ specifically for Erdős–Rényi graph with light-tailed and dense dynamics. 
In particular, the dense connectivity in \cite{xu2023decentralized} 
requires any pair of clients connects with a rather high probability (of order $\bo (1)$) at any time step,
 which limits practical applicability.
Besides,
their work covers some reward distributions that are heavier than the sub-Gaussian class, 
but still have finite moment-generating functions (MGFs).
In contrast, we impose no such assumptions on graph connectivity,
consider sparse graphs with only $\bo (M)$ total degree (instead of the $\bo (M^2)$ total degree in \cite{xu2023decentralized}),
and allow for a significantly more general class of reward distributions, potentially with infinite variance and lacking finite MGFs.
On the other hand, \cite{zhu2023distributed} achieves an regret bound of the same order, but assumes a connected or periodically connected graph under sub-Gaussian rewards, 
which may not hold for sparse graphs. 
Last but not least, our results close an open problem in \cite{dubey2020cooperative} regarding homogeneous rewards under heavy-tailed settings (Section \ref{sec:reg-homo}), thus bridging the gap in existing literature and addressing challenges posed by heavy-tailed graphs.






 
\section{Conclusion and Future Work}\label{sec:conclusion}

We characterize the multi-agent multi-armed bandit problem with heavy tails in both rewards and graphs to capture complex real-world scenarios. 
We consider a setting where $M$ clients have asymmetric and low degrees on graphs, and reward observations can deviate significantly from the mean. These complexities introduce challenges in both client communication and statistical inference of reward mean values, potentially leading to larger regret. Surprisingly, we prove that with novel algorithm designs, regret is sublinear in $M$ and $T$, with an (almost) order of $O(M^{1-\frac{1}{\alpha}} \log{T})$ and $O(M\log{T})$, in homogeneous and heterogeneous settings, respectively. 
These results improve the regret bounds in existing work that considers less complex settings.

Moving forward, it would be valuable to explore other types of random graphs and analyze how regret scales with different graph dynamics. Also, while our framework allows clients to share information about their neighbors, removing this assumption in future research would be exciting.

\bibliographystyle{abbrv} 
\bibliography{example_paper} 

\newpage
\appendix
\section{Pseudo Code}

\begin{algorithm}[H]
\SetAlgoLined
\caption{HT-HTUCB (Heavy-Tailed Heterogeneous UCB): Burn-in period}\label{alg:heter}
 Initialization: The length of the burn-in period is $L$ and we are also given $\tau_1 < L$;  In the time step $t = 0$, the estimates are initialized as $\bar{\mu}_i^m(0) = 0$, $n_{m,i}(0) = 0$, $\hat{\bar{\mu}}_{i,j}^m(0) = 0 $, and $P_0(m,j) = $ for any arm $i$ and clients $m, j$\par
\For{$\tau_1 < t \leq L$}{
  The environment generates a sample graph $G_t = (V, E_t)$\par
  \For{each client $m$}{
  Sample arm $a^m_t = (t \mod K)$\par
  Receive rewards $r_{a^m_t}^m(t)$ and update $n_{m,i}(t) = n_{m,i}(t-1) + \mathds{1}_{a_m^t = i}$\par
 Update the local estimates for any arm $i$: $\bar{\mu}_i^m(t) = \frac{n_{m,i}(t-1)\bar{\mu}_i^m(t-1) + r_{a^m_t}^m(t) \cdot 1_{a^m_t = i}}{n_{m,i}(t-1) + 1_{a^m_t = i}}$\par
     Update the maintained matrix $P_t(m,j) = \frac{(t-1)P_{t-1}(m,j)+X_{m,j}^t}{t}$ for each $j \in V$\par
   Send $\{\bar{\mu}_i^m(t)\}_{i =1}^{i=K}$ to all clients in $\mathcal{N}_m(t)$\par
   Receive $\{\bar{\mu}_i^j(t)\}_{i=1}^{i = K}$ from all clients $j \in \mathcal{N}_m(t)$ and store them as $\hat{\bar{\mu}}_{i,j}^m(t)$.
  }
 }
\For{each client $m$ and arm $i$}{For client $1 \leq j \leq M$, set
$h^L(m,j) = \max_{s\geq 1}\{(m,j) \in E_s\}$ or $0$ if such $s$ does not exist \\
$\Tilde{\mu}_i^m(L+1) = \sum_{j=1}^MP^{\prime}_{m,j}(L) \hat{\bar{\mu}}_{i,j}^m(h^L_{m,j})$ where $P^{\prime}_{m,j}(L) = \begin{cases}
     \frac{1}{M} & \text{if $P_L(m,j) > 0$} \\
     0 & \text{otherwise}
    \end{cases}$ \;
} 
\end{algorithm}

\begin{algorithm}[H]
\SetAlgoLined
\caption{HT-HMUCB (Heavy-Tailed Homogeneous UCB): Identification of Hub Center}\label{alg:find hub center}
 Initialization: The length of search period $L$; set $d_i(j) = -1$ for any $i,j \in [M]$\par
 \For{$t = 1$}{
    The environment generates graph $G_1 = (V, E_1)$ at time $t = 1$\par
    \For{each client $m \in [M]$}{
        Set $d_m(m) \leftarrow d_m$, where $d_m$ is the degree of $m$ on $G_1$ \par
    }
 }
\For{ t = 2,3,\ldots, L }{
    \For{each client $m \in [M]$}{
        If $d_j(m) \geq 0$ for some $j \in [M]$, send $d_j(m)$ to all clients on the neighbor set $\mathcal N_m(t)$ \par
        For any $d_j(i)$ received from a neighbor $i \in \mathcal N_m(t)$, if $d_j(m) = -1$, then set $d_j(m) \leftarrow d_j(i)$ \par
    }
}
\For{ each client $m \in [M]$ }{
    Set $\hat i(m) \leftarrow \arg\max_{i \in [M]} d_i(m)$; when the argument minimum is not unique, pick the smallest index \par
    If $\hat i(m) = m$, act as the \emph{hub center} from now on; otherwise, act as \emph{non-center}
}
\end{algorithm}

\section{Proof for Section~\ref{sec:graphs}}

We first set a few notations that will be used frequently in this section.
Let $\mathbb Z$ be the set of integers.
For any $x\in\mathbb{R}$, we use
$
{\floor{x}} \delequal \max\{ n \in \mathbb{Z}:\ n \leq x \},\
 {\ceil{x}} \delequal  \min\{n \in \mathbb{Z}:\ n \geq x\}
$
to denote the floor and ceiling function.
For any $x,y \in \R$, let ${x \wedge y} \delequal \min\{x,y\}$  and ${x\vee y} \delequal \max\{x,y\}$.

\subsection{Proof of Lemmas~\ref{lemma, hub: deterministic lower bound} and \ref{lemma, hub: stochastic lower bound}}

To prove Lemma~\ref{lemma, hub: deterministic lower bound},
we first prepare a few technical tools.
Let $H_M \delequal \max_{i \in [M]}h_i$,
where $h_i$'s are iid copies of some random variable $h$ taking values in $[0,\infty)$.
By imposing Assumption~\ref{assumption: heavy-tailed graph} on the law of $h$,
Lemma~\ref{lemma: extreme value in H} establishes a high-probability bound for  $H_M$.

\begin{lemma}\label{lemma: extreme value in H}
Let Assumption~\ref{assumption: heavy-tailed graph} hold.
For any $\Delta \in (0, 1/\alpha)$,
\begin{align*}
    \P\big(H_M \leq M^{\frac{1}{\alpha} - \Delta}\big)
    =
    \lo \Big( \exp\big(-M^{\alpha\Delta/2}\big) \Big)
    \qquad
    \text{as } M \to \infty.
\end{align*}
\end{lemma}

\begin{proof}
    We write $n(M) = M^{\frac{1}{\alpha} - \Delta}$ in this proof,
    and observe that
    \begin{align}
        \P(H_M \leq n(M))
        & = 
        \P(h_i \leq n(M)\ \forall i \in [M])
        \nonumber
        \\ 
        & = \Big( 1 - \P(h > n(M))\Big)^{ M }
        \qquad
        \text{by the iid nature of $h_i$'s};
        \nonumber
        \\ 
    \Longrightarrow
    \log \P(H_M \leq n(M))
    & = 
    M \cdot \log  \Big( 1 - \P(h > n(M))\Big)
    \nonumber
    \\ 
    & \leq 
    -M\cdot \P(h > n(M)).
    \label{proof: intermedaite result 1, lemma: extreme value in H}
    \end{align}
The inequality follows from $\log(1 - x) \leq -x$ for all $x \in (0,1)$.
Due to our choice of $\Delta > 0$, we can fix $\epsilon > 0$ small enough such that
$
(\frac{1}{\alpha} - \Delta)(\alpha + \epsilon) < 1 - \frac{\alpha\Delta}{2}.
$
By Potter's bound for regularly varying functions (see, e.g., Proposition~2.6 of \cite{resnick2007heavy}),
it holds for all $x$ large enough that
$
\P(h > x) \geq x^{ -(\alpha + \epsilon)  }.
$
Therefore, for any $M$ large enough, we have
\begin{align*}
    \P(h > n(M))
    \geq 
    M^{ - (\frac{1}{\alpha } - \Delta)(\alpha + \epsilon) }
    >
    M^{ -1 + \frac{\alpha \Delta}{2} }.
\end{align*}
As a result, we have in \eqref{proof: intermedaite result 1, lemma: extreme value in H} that
\begin{align*}
    \P(H_M \leq n(M))
    & = 
    \lo \Big(\exp\big( -M^{ \alpha\Delta/2  } \big)\Big)
\end{align*}
for all $M$ large enough. This concludes the proof.
\end{proof}

Let $\mathbb Z_+ = \{0,1,2,\ldots\}$ be the set of non-negative integers.
Let $d^t_i$ be the degree of the $i^\text{th}$ node over graph $G_t$.
To lighten notations we write $d_i = d^1_i$ at time $t = 1$.
Under Assumption~\ref{assumption: lower bound for h}, Lemma~\ref{lemma: unlikely for small h to have high degree} provides useful bounds for the conditional law of the degree $d_i$.

\begin{lemma}\label{lemma: unlikely for small h to have high degree}
    Let Assumption~\ref{assumption: lower bound for h} hold.
    Let $\theta = \E h$ in \eqref{def: kernel for the random graph}.
For any $M \in \mathbb Z_+$, $\beta \in (0,1)$ and $\Delta \in (0,\beta)$,
\begin{align}
    \max_{i \in [M]}\P\big(d_i \geq M^{\beta + \Delta}\ \big|\ h_i \leq M^{\beta}\big) = 
    \lo \big( \exp(-M^{\beta - \Delta})\big),
    \label{claim 1, lemma: unlikely for small h to have high degree}
    \\ 
    \max_{i \in [M]}\P\big(d_i \leq M^{\beta - \Delta}\ \big|\ h_i \geq M^{\beta}\big) = 
    \lo \big( \exp(-M^{\beta - \Delta})\big).
    \label{claim 2, lemma: unlikely for small h to have high degree}
\end{align}
\end{lemma}

\begin{proof}
\textbf{Proof of Claim \eqref{claim 1, lemma: unlikely for small h to have high degree}}.
    Since the marginal distribution of $\big(h_i,d_i)$ is identical for each $i \in [M]$,
    it suffices to show that 
    \begin{align}
    \P\big(d_1 \geq M^{\beta + \Delta}\ \big|\ h_1 \leq M^{\beta} \big) = 
    \lo \big( \exp(-M^{\beta - \Delta}) \big).
        \nonumber
    \end{align}
On event $\{ h_1 \leq M^\beta \}$, by the definition of kernel $P(\cdot,\cdot)$ in \eqref{def: kernel for the random graph}, we have
\begin{align}
    d_1|\{ h_1 \leq M^\beta \} \stleq
    \sum_{ i = 2 }^{M}
    \underbrace{
    \text{Bernoulli}\bigg( \frac{ h_i M^{\beta} }{\theta M} \wedge 1  \bigg)
    }_{ = V_i }.
    \nonumber
\end{align}
Here, $X \stleq Y$ refers the the stochastic dominance where $\P(X > z) \leq \P(Y > z)\ \forall z \in \R$,
and $X|A$ denote the conditional law $\P(Z > x) = \P(X > x |A)$.
Besides, note that $V_i$'s are iid, with $0\leq V_i \leq 1$ and (recall that $\theta = \E h_i$)
\begin{align}
    \E V_i & = \E\bigg( \frac{ h_i M^{\beta} }{\theta M} \wedge 1  \bigg) 
    \leq 
    \E\bigg( \frac{ h_i M^{\beta} }{\theta M}  \bigg) 
    =
    \frac{M^\beta}{\theta M} \cdot \E h_i = \frac{M^\beta}{M},
    \nonumber
    \\
    \Longrightarrow
    \E \bigg[ \sum_{i = 2}^M V_i \bigg]
    & = (M-1)\E V_2 \leq M^\beta < \frac{1}{2}M^{\beta +\Delta}
    \qquad
    \text{for any $M$ large enough.}
    \label{proof: upper bound for expectation, lemma: unlikely for small h to have high degree}
\end{align}
On the other hand, by Assumption~\ref{assumption: lower bound for h},
\begin{align}
    \E V_i & \geq \frac{c_h M^\beta}{\theta M}
    \qquad
    \text{ for all $M$ large enough due to $\beta < 1$}
    \nonumber
    \\
    \Longrightarrow
    \E\bigg[ \sum_{i = 2}^M V_i \bigg]
    & \geq 
    (M - 1) \cdot \frac{c_h M^\beta}{M}
    \geq \frac{c_h}{2}M^\beta 
    \qquad
    \text{for any $M$ large enough}.
    \label{proof: lower bound for expectation, lemma: unlikely for small h to have high degree}
\end{align}
Therefore, for any $M$ sufficiently large,
\begin{align*}
    \P\Bigg( \sum_{i = 2}^M V_i \geq M^{\beta + \Delta} \Bigg)
    & \leq 
     \P\Bigg( \sum_{i = 2}^M V_i \geq 2 \bigg[ \E \sum_{i = 2}^M V_i \bigg] \Bigg)
     \qquad
     \text{by \eqref{proof: upper bound for expectation, lemma: unlikely for small h to have high degree}}
     \\ 
     & \leq 
     \exp\bigg( -\frac{1}{3} \E \bigg[ \sum_{i = 2}^M V_i \bigg]  \bigg)
     \qquad
     \text{by Chernoff bound}
     \\ 
     & \leq 
     \exp\bigg( -\frac{1}{6}c_h M^\beta \bigg)
     \qquad
     \text{by \eqref{proof: lower bound for expectation, lemma: unlikely for small h to have high degree}}
     \\ 
     & = 
     \lo \Big( \exp(-M^{\beta - \Delta}) \Big).
\end{align*}
This concludes the proof of \eqref{claim 1, lemma: unlikely for small h to have high degree}.

\smallskip
\noindent
\textbf{Proof of Claim \eqref{claim 2, lemma: unlikely for small h to have high degree}}.
Analogously, it suffices to prove the claim for $i = 1$.
Due to $\beta < 1$, it holds for any $M$ large enough that
$
\frac{c_h M^\beta}{\theta M} < 1.
$
Let $V_i$'s be iid copies of Bernoulli$( \frac{c_h M^\beta}{\theta M} )$.
On event $\{h_1 \geq M^\beta\}$,
it follows from Assumption~\ref{assumption: lower bound for h}
and the definition of kernel $P(\cdot,\cdot)$ in \eqref{def: kernel for the random graph} that
\begin{align*}
    \sum_{ i = 2 }^M V_i \stleq d_1|\{h_1 \geq M^\beta\}.
\end{align*}
Furthermore,
\begin{align*}
    \E\bigg[ \sum_{ i = 2 }^M V_i \bigg]
    &\geq (M-1) \cdot \frac{c_h M^\beta}{\theta M}
    \geq 
    \frac{c_h}{2\theta}M^\beta
    \qquad
    \text{for all $M \geq 2$}
    \\ 
    & \geq 2 M^{\beta -\Delta}\qquad
    \text{ for any $M$ large enough.}
\end{align*}
As a result, for any $M$ sufficiently large,
\begin{align*}
    \P\Bigg(
        \sum_{ i = 2}^M V_i \leq M^{\beta - \Delta}
    \Bigg)
    & \leq 
    \P\Bigg(
    \sum_{ i = 2}^M V_i \leq \frac{1}{2}\E\bigg[ \sum_{ i = 2}^M V_i \bigg]
    \Bigg)
    \\ 
    & \leq 
    \exp\bigg(
        -\frac{1}{4}\E\bigg[ \sum_{ i = 2}^M V_i \bigg]
    \bigg)
    \qquad\text{by Chernoff bound}
    \\ 
    & \leq 
    \exp\bigg(
        -\frac{c_h}{8\theta}M^\beta
    \bigg)
    \\ 
    & = \lo \big( \exp(-M^{\beta - \Delta}) \big).
\end{align*}
This concludes the proof of \eqref{claim 2, lemma: unlikely for small h to have high degree}.
\end{proof}

Recall the definition of 
\begin{align}
    \hat i \delequal \arg\max_{i \in [M]}d_i,
    \label{proof, def: hat i, node with the highest degree on random graph}
\end{align}
which is the node with the highest degree over graph $G_1$ (i.e., at time $t = 1$).
When there are ties, we arbitrarily pick one of argument maximum as $\hat i$.
As an immediate consequence from Lemma~\ref{lemma: unlikely for small h to have high degree},
the next Lemma shows that the node $\hat i$ will almost always have a large weight $h_{\hat i}$; in other words, the node with empirically largest degree (at time $t = 1$) will almost always have a large attraction weight.

\begin{lemma}\label{lemma: h i for i with the largest degree}
 Let Assumptions~\ref{assumption: heavy-tailed graph} and \ref{assumption: lower bound for h} hold.
Define event
\begin{align}
    B(M,\Delta)
    =
    \big\{
        h_{\hat i} > M^{ \frac{1}{\alpha } - \Delta }
    \big\}.
    \nonumber
\end{align}
For any $\Delta \in (0,\frac{1}{2\alpha })$, and $\gamma > 0$ small enough such that
\begin{align}
    \gamma < \frac{\alpha  \Delta}{4},
    \qquad
    \gamma < \frac{1}{\alpha } - 2\Delta,
    \label{proof: choice of gamma, lemma: h i for i with the largest degree}
\end{align}
we have (as $M \to \infty$)
\begin{align*}
    \P\Big( B(M,\Delta)^\complement \Big)
    =
    \lo \Big(
        \exp(-M^\gamma)
    \Big).
\end{align*}
\end{lemma}

\begin{proof}
In this proof, let $i^* = \arg\max_{i \in [M]}h_i$
denote the index of the node with largest weight $h_i$.
Again, when there are ties, we arbitrarily pick one of such $i^*$'s.
Note that on event 
$
\big\{ 
        d_i < M^{ \frac{1}{\alpha } - \frac{\Delta}{2} }\ \forall i\in[M]\text{ with }h_i < M^{ \frac{1}{\alpha } - \Delta }
    \big\},
$
the claim $h_{\hat i} \geq M^{ \frac{1}{\alpha} - \Delta }$ must hold
if, for some $i \in [M]$ with  $d_i \geq M^{ \frac{1}{\alpha} - \frac{\Delta}{2} }$,
we have $h_i \geq M^{ \frac{1}{\alpha } - \Delta }$.
In particular, note that
\begin{align*}
    B(M,\Delta)
    & \supseteq
    \Big\{ 
        h_{i^*} \geq M^{\frac{1}{\alpha } - \frac{\Delta}{2}},\ 
        d_{i^*} \geq M^{\frac{1}{\alpha } - \Delta}
    \Big\}
    \cap 
    \Big\{ 
        d_i < M^{ \frac{1}{\alpha } - \frac{\Delta}{2} }\ \forall i\in[M]\text{ with }h_i < M^{ \frac{1}{\alpha } - \Delta }
    \Big\},
\end{align*}
which leads to the upper bound
\begin{align*}
    \P\Big( B(M,\Delta)^\complement \Big)
    & \leq 
    \underbrace{ \P\big(h_{i^*} \leq M^{\frac{1}{\alpha } - \frac{\Delta}{2}}\big) }_{ = p_1(M,\Delta) }
    +
    \underbrace{
    \P\Big(
        d_{i^*} < M^{\frac{1}{\alpha } - \Delta}\ \Big|\ 
        h_{i^*} \geq M^{\frac{1}{\alpha } - \frac{\Delta}{2}}
        \Big)
    }_{ p_2(M,\Delta) }
    \\ 
    & \qquad
    +
    M \cdot 
    \underbrace{
        \P\Big(
            d_1 > M^{ \frac{1}{\alpha } - \frac{\Delta}{2} }\ \Big|\ 
            h_1 < M^{ \frac{1 }{\alpha } - \Delta }
        \Big)
    }_{ p_3(M,\Delta) }.
\end{align*}
By Lemma~\ref{lemma: extreme value in H},
$
p_1(M,\Delta) = \lo \big( \exp( -M^{ \alpha  \Delta/4 } ) \big).
$
By Lemma~\ref{lemma: unlikely for small h to have high degree},
we get
$
p_2(M,\Delta) = \lo \big( \exp(- M^{ \frac{1}{\alpha } - \Delta } ) \big)
$
and
$
p_3(M,\Delta) = \lo \big( \exp(- M^{ \frac{1}{\alpha } - 2\Delta } ) \big).
$
By our choice of $\gamma > 0$ in \eqref{proof: choice of gamma, lemma: h i for i with the largest degree}, we conclude the proof.
\end{proof}

Recall that $\theta = \E h$.
Let
\begin{align}
    S_{*,\Delta}
    =
    \big\{
        i \in [M]:\ 
        h_i > \theta M^{ 1 +\Delta - \frac{1}{\alpha } }
    \big\}
    \label{proof, def: set S * Delta, the hub, set of nodes with large h}
\end{align}
be the collection of nodes with large weights $h_i$ w.r.t.\ threshold $\theta M^{ 1 +\Delta - \frac{1}{\alpha }}$.
The next lemma develops high-probability bounds for the size of $S_{*,\Delta}$.

\begin{lemma}\label{lemma: size of hub S*}
    Let Assumption~\ref{assumption: heavy-tailed graph} hold
    {with $\alpha  \in (1,2)$.}
For any $\zeta \in (0, 2 - \alpha )$, any $\Delta > 0$ small enough such
\begin{align}
    (\alpha  + \Delta)(1 + \Delta - \frac{1}{\alpha }) \leq \alpha  - 1 + \frac{\zeta}{3},
    \label{proof: choice of Delta, lemma: size of hub S*}
\end{align}
we have
\begin{align*}
    \P\big( | S_{*,\Delta} | \leq M^{ 2 - \alpha  - \zeta } \big)
    =
    \lo \Big( \exp(-M^{2 - \alpha  - \zeta}) \Big)
    \qquad
    \text{ as }M \to \infty.
\end{align*}
\end{lemma}

\begin{proof}
We 
write $n(M) =  \theta M^{ 1 +\Delta - \frac{1}{\alpha }}$
and observe that
$
|S_{*,\Delta}| \distequal 
\text{Binomial}\big(M, \P(h > n(M))\big). 
$
Let $V_i$'s be iid copies of 
$
\text{Bernoulli}\big(\P(h > n(M))\big).
$
Observe that
\begin{align*}
    \E\bigg[ \sum_{i = 1}^M V_i \bigg]
    & = M \P(h > n(M))
    \\ 
    & \geq 
    M \cdot \frac{1}{ \big(n(M)\big)^{ \alpha  + \Delta  } }
    \qquad
    \text{for any $M$ large enough due to Potter's bound 
    (Proposition~2.6 of \cite{resnick2007heavy})}
    \\ 
    & \geq \frac{1}{\theta^{ \alpha  + \Delta }} \cdot \frac{M}{ M^{ \alpha  -1 + \frac{\zeta}{3} } }
    \qquad
    \text{by the choice of $\Delta$ in \eqref{proof: choice of Delta, lemma: size of hub S*} and the definition $n(M) =  \theta M^{ 1 +\Delta - \frac{1}{\alpha }}$}
    \\ 
    & =
    \frac{1}{\theta^{ \alpha  + \Delta }} \cdot  M^{ 2 - \alpha  - \frac{\zeta}{3} }
    \\ 
    & \geq 2M^{2 - \alpha  - \frac{2\zeta}{3}}
    \qquad
    \text{for any $M$ large enough}.
\end{align*}
Therefore, for such large $M$,
\begin{align*}
    \P\big( | S_{*,\Delta} | \leq M^{ 2 - \alpha  - \zeta } \big)
    & \leq 
    \P\Bigg( \sum_{i = 1}^M V_i \leq \frac{1}{2}\E\bigg[\sum_{i = 1}^M V_i\bigg] \Bigg)
    \\ 
    & \leq 
    \exp\bigg( - \frac{1}{4}\E\bigg[\sum_{i = 1}^M V_i\bigg] \bigg)
    \qquad\text{by Chernoff bound}
    \\ 
    & 
    \leq 
    \exp\bigg( - \frac{1}{2}M^{2 - \alpha  - \frac{2\zeta}{3}}\bigg)
    \\ 
    & = \lo \bigg( \exp(-M^{ 2-\alpha  - \zeta }) \bigg).
\end{align*}
This concludes the proof.
\end{proof}

Now, we are ready to prove Lemma~\ref{lemma, hub: deterministic lower bound}.

\begin{lemma*}[Lemma \ref{lemma, hub: deterministic lower bound}]
Let Assumptions~\ref{assumption: heavy-tailed graph} and \ref{assumption: lower bound for h} hold with $\alpha \in (1,2)$.
Given $\zeta \in (0,2 - \alpha)$, there exists $\gamma > 0$ such that
\begin{align*}
    \P\Big( |S_0| \leq M^{ 2 - \alpha - \zeta} \Big)
    =
    \lo\Big(\exp( -M^\gamma)\Big).
\end{align*}
\end{lemma*}

\begin{proof}[Proof of Lemma \ref{lemma, hub: deterministic lower bound}]
Let $\hat i$ and $S_{*,\Delta}$ be defined as in \eqref{proof, def: hat i, node with the highest degree on random graph} and \eqref{proof, def: set S * Delta, the hub, set of nodes with large h}, respectively.
Recall that $S^t_0 = \{ i \in [M]: (i,\hat i) \in E_t \}$ is the set of agents communicating with $\hat i$ at time $t$.
Note that
by the definition of the kernel $P(\cdot,\cdot)$ in \eqref{def: kernel for the random graph},
on event $\{ h_{\hat i} > M^{  \frac{1}{\alpha}  - \Delta } \}$
we must have $S_{*,\Delta} \subseteq S_0$ for any $t \geq 1$.
As a result, we have
\begin{align*}
    \big\{ |S_0| > M^{ 2 - \alpha - \zeta  } \big\}
    \supseteq
    \{ h_{\hat i} > M^{  \frac{1}{\alpha}  - \Delta } \}
    \cap 
    \{ |S_{*,\Delta}| > M^{2 - \alpha - \zeta} \},
\end{align*}
and hence
\begin{align*}
    \P\Big( |S_0| \leq M^{ 2 - \alpha - \zeta } \Big)
    \leq 
    \P\big( h_{\hat i} \leq M^{  \frac{1}{\alpha}  - \Delta } \big) + 
    \P\big( |S_{*,\Delta}| \leq M^{2 - \alpha - \zeta} \big).
\end{align*}
Now, we pick 
$
\Delta \in (0,\frac{1}{2\alpha})
$
small enough such that condition \eqref{proof: choice of Delta, lemma: size of hub S*} holds,
and then pick
$\gamma \in (0, 2 - \alpha - \zeta)$ small enough such that condition \eqref{proof: choice of gamma, lemma: h i for i with the largest degree} holds.
By Lemma~\ref{lemma: h i for i with the largest degree},
$
\P\big( h_{\hat i} \leq M^{  \frac{1}{\alpha}  - \Delta } \big) =
\lo \Big(
        \exp(-M^\gamma)
    \Big).
$
By Lemma~\ref{lemma: size of hub S*},
$
 \P\big( h_{\hat i} \leq M^{  \frac{1}{\alpha}  - \Delta } \big) + 
    \P\big( |S_{*,\Delta}| \leq M^{2 - \alpha - \zeta} \big)
    =
    \lo \Big( \exp(-M^{2 - \alpha  - \zeta}) \Big)
    =
    \lo \Big( \exp(-M^{\gamma}) \Big),
$
where the last step follows from our choice of $\gamma \in (0, 2 - \alpha - \zeta)$.
This concludes the proof.
\end{proof}

Now, we provide the proof of  Lemma \ref{lemma, hub: stochastic lower bound}.
For any $p \in (0,1)$, we say that $X$ is Geom$(p)$ if
\begin{align*}
    \P(X > k) = (1 - p)^k\qquad\forall k = 1,2,\ldots.
\end{align*}


\begin{lemma*}[Lemma~\ref{lemma, hub: stochastic lower bound}]
Let Assumptions~\ref{assumption: heavy-tailed graph} and \ref{assumption: lower bound for h}  hold.
Let $\zeta \in (0,1 - \frac{1}{\alpha})$.
There exists $\gamma > 0$ such that 
\begin{align}
    \P\big(h_{\hat i} < M^{ \frac{1}{\alpha} - \frac{\zeta}{2} }\big) = 
\lo \big( \exp(-M^\gamma ) \big).
\label{claim 1, lemma, hub: stochastic lower bound}
\end{align}
Furthermore, 
there exists $M_0 > 0$ such that
\begin{align}
    \P\bigg( \sup_{t \leq T}t - \tau(t) > \log T\ \bigg|\ h_{\hat i} \geq M^{ \frac{1}{\alpha} - \frac{\zeta}{2} }\bigg) 
    \leq 
    \frac{1}{MT}
    \qquad
    \forall M \geq M_0,\ T \geq 1,
    \label{claim 2, lemma, hub: stochastic lower bound}
\end{align}
where
$
\tau(t) = \max\{u \leq t:\ |S^u_0| > M^{ \frac{1}{\alpha} - \zeta}\}.
$
\end{lemma*}

\begin{proof}[Proof of Lemma~\ref{lemma, hub: stochastic lower bound}]
Claim~\eqref{claim 1, lemma, hub: stochastic lower bound} is exactly the content of Lemma~\ref{lemma: h i for i with the largest degree}.
The rest of the proof is devoted to establishing the claim~\eqref{claim 2, lemma, hub: stochastic lower bound}.
To proceed, let $T_0 = 0$, and 
\begin{align*}
    T_k \delequal \min\big\{
        t > T_{k-1}:\ |S^t_0|  > M^{ \frac{1}{\alpha} - \zeta}
    \big\}
\end{align*}
be the $k^\text{th}$ time the hub around $\hat i$ is large (w.r.t.\ threshold $M^{ \frac{1}{\alpha} - \epsilon   }$).
Let $V_k \delequal T_k - T_{k - 1}$ be the $k^\text{th}$ time gap between two large hubs.
Note that 
\begin{align}
    \sup_{t \leq T}t - \tau(t)
    \leq 
    \max\bigg\{
        V_k - 1:\ 
        k \geq 1,\ 
        \sum_{i = 1}^{k - 1} V_i \leq T
    \bigg\}
    \leq 
    \max_{k \leq T}V_k - 1.
    \label{proof, upper bound for the target term, lemma, hub: stochastic lower bound}
\end{align}
To proceed, note that 
$
\sup_{ t \geq 1  }
\P\big( |S^t_0| \leq M^{ \frac{1}{\alpha} - \zeta  } \ \big|\ h_{\hat i} \geq M^{ \frac{1}{\alpha} - \frac{\zeta}{2}  }  \big)
=
\P\big( |S^1_0| \leq M^{ \frac{1}{\alpha} - \zeta   }\ \big|\ h_{\hat i} \geq M^{ \frac{1}{\alpha} - \frac{\zeta}{2}  }   \big),
$
since the sequence $S^t_0$'s are independent across $t$ when conditioned on the value of $h_{\hat i}$.
Now, we fix some $\tilde \gamma \in (0, \frac{1}{\alpha} - \zeta)$.
By Lemma~\ref{lemma: unlikely for small h to have high degree},
there exists $M_0$ such that
\begin{align*}
    \sup_{ t \geq 1  }
\P\Big( |S^t_0| \leq M^{ \frac{1}{\alpha} - \zeta   } \ \Big|\ h_{\hat i} \geq M^{ \frac{1}{\alpha} - \frac{\zeta}{2}  }  \Big)
    \leq 
    \exp(
        -M^{ \tilde \gamma  }
    )
    \qquad
    \forall M \geq M_0.
\end{align*}
As a result,
there exists a coupling between $(V_i)_{i \leq T}$ and $(\tilde V_i)_{i \leq T}$, which are iid copies of Geom$\big( 1 - \exp(-M^{ \tilde \gamma  })  \big)$, such that 
\begin{align*}
    (V_1,\ldots,V_T)|\{ h_{\hat i} \geq M^{ \frac{1}{\alpha} }   \} \stleq 
    (\tilde V_1,\ldots, \tilde V_T),
\end{align*}
Together with the upper bound in \eqref{proof, upper bound for the target term, lemma, hub: stochastic lower bound}, we yield (for any $M \geq M_0$)
\begin{align*}
    \P\bigg( \sup_{t \leq T}t - \tau(t) > \log T\ \bigg|\ h_{\hat i} \geq M^{ \frac{1}{\alpha} - \frac{\zeta}{2} }\bigg) 
    & \leq 
    T \cdot \P\big(\tilde V_1 - 1 > \log T \big)
    =
    T \cdot \big( \exp(
        -M^{ \tilde \gamma  }
    )  \big)^{1 + \log T}.
\end{align*}
Note that $\exp(
        -M^{ \tilde \gamma  } = \lo(M^{-2})$.
By picking a larger $M_0$ if necessary, we can ensure that 
$
\exp(
        -M^{ \tilde \gamma  }  \leq M^{-2}
$
for each $M \geq M_0$, and hence
\begin{align*}
     \P\bigg( \sup_{t \leq T}t - \tau(t) > \log T\ \bigg|\ h_{\hat i} \geq M^{ \frac{1}{\alpha} - \frac{\zeta}{2} }\bigg) 
     & \leq 
     T / M^{ 2 \cdot (1 + \log T) }
     \leq 
     \frac{1}{M} \cdot \frac{T}{ M^{ 2\log T  }  }.
\end{align*}
Lastly, by picking an even larger $M_0$ if needed, we ensure that $M_0 \geq e$,
so for each $M \geq M_0$ and $T \geq 1$, we have 
\begin{align*}
    \P\bigg( \sup_{t \leq T}t - \tau(t) > \log T\ \bigg|\ h_{\hat i} \geq M^{ \frac{1}{\alpha} - \frac{\zeta}{2} }\bigg) 
    \leq 
    \frac{1}{M} \cdot \frac{T}{ e^{ 2\log T } } = \frac{1}{M} \cdot \frac{T}{T^2} = \frac{1}{MT},
\end{align*}
which concludes the proof.
\end{proof}

\subsection{Proof of Lemma~\ref{lemma, information delay over sparse graphs}}

Central to our proof is the following stochastic dominance argument regarding the graphs $(G_t)_{ t \geq 1 }$.
Specifically, 
given some non-empty subset of clients $S \subseteq [M]$,
we
recall the definitions of $\bar S^0 = S$ and 
\begin{align}
    \bar S^t \delequal 
\{
i \in [M]:\ 
i \in \bar S^{t - 1};\text{ or }\exists j \in \bar S^{t - 1}\text{ s.t. }(i,j) \in E_t
\},
\nonumber
\end{align}
which represents the collection of clients that have received the message at time $t$, which was sent from $S$ at time $1$ and is passed to neighbors over graph $G_u$ at each time $u \leq t$.
Let $K \delequal |S|$ be the count of nodes in $S$ and, without loss of generality, write $S = \{i_1,i_2,\ldots,i_K\}$.
Let
\begin{align*}
    \bar V^1_{i_1} \delequal 
    \{
    i \in [M]:\ i \notin S,\ (i,i_1) \in E_1
    \}.
\end{align*}
be the set of nodes that are outside of $S$ but can be reached from $i_1$ within one step (i.e., they are neighbors of $i_1$ over the graph $G_1$ at time $t = 1$).
Analogously, let
\begin{align*}
    \bar V^1_{i_k} \delequal \{ i \in [M]:\ i \notin S,\ (i,i_k) \in E_1  \} \setminus
    \bigg( \bigcup_{ j \in [k-1]}\bar V^1_{i_j} \bigg)
    \qquad \forall k = 2,3,\ldots,K
\end{align*}
be the set of nodes that are outsides of $S$ and can be reached by $i_k$ (but not any $i_j$ with $j \in [k-1]$) within one step.
By definition, we have 
\begin{align*}
    \bar V^1_{i_j} \cap \bar V^1_{i_k} = \emptyset\ \forall j \neq k,
    \qquad
    \bar S^1 = \bigcup_{ k \in [K] }\Big( \{i_k\} \cup \bar V^1_{i_k}  \Big).
\end{align*}
We first consider the case where $K < M/2$, 
and we are able to uniformly randomly pick $\ceil{M/2}$ nodes that are outside of $S$.
By only checking whether these nodes are connected to $i_1$, 
and due to the lower bound $c_h > 0$ in Assumption~\ref{assumption: lower bound for h},
 as well as the definition of the kernel $P(h,h^\prime)$ in \eqref{def: kernel for the random graph},
we have 
\begin{align}
    Z_M(c_h) \stleq |\bar V^1_{i_1}|
    \qquad
    \text{where }Z_M(c_h) \distequal \text{Binomial}\Big( \ceil{M/2}, \frac{c_h^2}{\theta M}   \Big).
    \label{def: Z_M(c), stochastic dominance}
\end{align}
Again,
for any random variables $X$ and $Y$,
we use $X\stleq Y$ to denote stochastic dominance between $X$ and $Y$,
in the sense that
$
\P(X\geq x) \leq \P(Y \geq x)
$
holds
for any $x \in \R$.
Furthermore, consider the following inductive procedure for each $k = 2,3,\ldots,K$:
On the event 
\begin{align}
    \Bigg\{|S| + \sum_{ j \in [k-1] }| \bar V^1_{i_j}  | < M/2 \Bigg\},
    \label{def: event, connected nodes less than M/2}
\end{align}
we are, again, able to uniformly randomly pick $\ceil{M/2}$ nodes that are still outside of $S$ and $\bigcup_{j \in [k-1]}\bar V^1_{i_j}$.
Let $Z^{(k)}_M(c_h)$ be iid copies of $Z_M(c_h)$.
By repeating the arguments above, on the event defined in \eqref{def: event, connected nodes less than M/2},
we have
\begin{align}
    Z^{(k)}_M(c_h) \stleq |\bar V^1_{i_k}|\qquad\forall k = 1,2,\ldots,K.
    \label{proof: stochastic dominance, step k}
\end{align}
Define a branching process (i.e., Galton-Watson process) $\big( X_t^S \big)_{t \geq 0}$ 
by
\begin{align}
    X^S_0 = |S|,
    \qquad
    X^S_t = \sum_{ i = 1 }^{ X^S_{t-1} }\Big( 1 + Z^{(t,i)}_M(c_h)\Big)\ \ \forall t \geq 1,
    \label{def: branching process X S t}
\end{align}
where 
$
Z^{(n,i)}_M(c_h)
$
are iid copies of $Z_M(c_h)$.
By the arguments in  \eqref{def: Z_M(c), stochastic dominance}--\eqref{proof: stochastic dominance, step k},
\begin{align}
    \min\bigg\{ X^S_t, \frac{M}{2} \bigg\} \stleq 
    \min\bigg\{ |\bar S^t|, \frac{M}{2} \bigg\}
    \qquad \forall t =0,1,2,\ldots.
    \label{property: stochastic dominance by branching process}
\end{align}

This coupling is crucial to our proof below.
Specifically, we fix a few constants.
Let
\begin{align}
    \rho_h \delequal \frac{c^2_h}{2\theta}. 
    \label{def: constant, rho h}
\end{align}
where $\theta = \E h$, and $c_h > 0$ is the constant lower bound for the law of $h$ stated in Assumption~\ref{assumption: lower bound for h}.
Next, pick $\gamma_h \in (0,1)$ such that
\begin{align}
    (1+\rho_h)\gamma_h \delequal 1 + \frac{\rho_h}{2}.
    \label{def: constant, gamma h}
\end{align}
Note that $\rho_h$ and $\gamma_h$ only depend on $\theta = \E h$ and the constant $c_h$ in Assumption~\ref{assumption: lower bound for h}.
That is, these constants only depend on the law of $h$, and do not vary with any other parameters.
Let
\begin{align}
    T_M \delequal \min\{ t \geq 1:\ |\bar S^t| \geq M/2 \}
    \label{def: stopping time T M, half coverage}
\end{align}
be the first time that at least half of the nodes have received the message sent out from $S$.

\begin{lemma}\label{lemma: geometric expansion for communication}
Under Assumption~\ref{assumption: lower bound for h},
it holds for any non-empty $S \subseteq \{1,2,\ldots,M\}$ that
\begin{align*}
    \P\big( T_M > \ceil{\tilde t_M}\big) \leq 
    \frac{2}{2 + \rho_h},
\end{align*}
where
\begin{align}
    \tilde t_M
    \delequal 
    \frac{
        \log M - \log\big(2(1-\sqrt{\gamma_h})\big)
    }{
        \log\big( 1 + \ceil{\frac{M}{2}} \cdot \frac{c^2_h}{\theta M} \big)
    }.
    \label{def: tilde t M, lemma: geometric expansion for communication}
\end{align}

\end{lemma}

\begin{proof}
For the branching process defined in \eqref{def: branching process X S t},
we use $(X_n)_{n \geq 0}$ to denote the process under the initial value $X_0 = 1$.
For any $t \geq 1$ such that $\E X_t > \frac{M}{2}$, observe that
\begin{align}
    \P(T_M > t)
    & = \P( |\bar S^t| < M/2)
    \nonumber
    \\ 
    & \leq \P( X^S_t < M/2)
    \qquad
    \text{by \eqref{property: stochastic dominance by branching process}}
    \nonumber
    \\ 
    & \leq \P(X_t < M/2)
    \leq 
    \P\bigg(
        | X_t - \E X_t | \geq \Big| \E X_t - \frac{M}{2} \Big|
    \bigg)
    \nonumber
    \\
    & \leq 
    \frac{
        \text{var}(X_t)
    }{
        \big|\E X_t -\frac{M}{2} \big|^2
    }
    \qquad
    \text{by Chebyshev's inequality}.
    \label{proof: ineq, lemma: geometric expansion for communication}
\end{align}
Let
\begin{align}
    x \delequal \ceil{\frac{M}{2}} \cdot \frac{c^2_h}{\theta M},
    \qquad
    m \delequal 1 + x,
    \qquad
    \sigma \delequal \sqrt{x \cdot \bigg(1 - \frac{c^2_h}{\theta M}\bigg)} \leq \sqrt{x}.
    \label{proof: notations, lemma: geometric expansion for communication}
\end{align}
By Theorem 5.1 of \cite{harris1963theory},
\begin{align*}
    \E X_t = m^t,
    \qquad
    \text{var}(X_t) \leq \frac{ \sigma^2 m^{2t} }{ m^2 - m }
    \qquad
    \forall t \geq 1.
\end{align*}
In particular, at $t = \ceil{\tilde t_M}$ (see \eqref{def: tilde t M, lemma: geometric expansion for communication}),
we have
\begin{align*}
    \log \E X_{ \ceil{\tilde t_M} } & \geq 
    \frac{
        \log M - \log\big(2(1-\sqrt{\gamma_h})\big)
    }{
        \log( 1 + x)
    }
    \cdot \log( 1 + x)
    = 
    \log\bigg(
        \frac{M}{ 2(1-\sqrt{\gamma_h}) }
    \bigg)
    \\ 
    \Longrightarrow
    \E X_{ \ceil{\tilde t_M} } & \geq 
    \frac{M}{2} \cdot \frac{1}{ 1 - \sqrt{\gamma_h} }
    \\ 
    \Longrightarrow
     \E X_{ \ceil{\tilde t_M} } - \frac{M}{2}
     & \geq 
     m^{ \ceil{\tilde t_M} }\cdot \big[ 1 - (1 - \sqrt{\gamma_h})\big] = m^{ \ceil{\tilde t_M} }\cdot\sqrt{\gamma_h}
     \\ 
     \Longrightarrow
     \bigg| \E X_{ \ceil{\tilde t_M} } - \frac{M}{2} \bigg|^2
     & \geq 
     \gamma_h \cdot  m^{ 2\ceil{\tilde t_M} }.
\end{align*}
Plugging these bounds back into \eqref{proof: ineq, lemma: geometric expansion for communication},
we yield (under $t = \ceil{\tilde t_M}$)
\begin{align*}
    \P(T_M > \ceil{\tilde t_M})
    & \leq 
    \frac{ \sigma^2 m^{2\ceil{\tilde t_M}} }{ m^2 - m }
    \cdot 
    \frac{1}{
        \gamma_h \cdot  m^{ 2\ceil{\tilde t_M} }
    }
    =
    \frac{\sigma^2}{ m(m - 1) } \cdot \frac{1}{\gamma_h}
    \\ 
    & \leq 
    \frac{x}{x(1+x)}\cdot \frac{1}{\gamma_h}
    =
    \frac{1}{1 + x}\cdot \frac{1}{\gamma_h}
    \\ 
    &
    \leq 
    \frac{1}{1 + \rho_h}\cdot \frac{1}{\gamma_h}
    \qquad
    \text{by definition of $x$ in \eqref{proof: notations, lemma: geometric expansion for communication}}
    \\
    & = \frac{2}{2 + \rho_h}
    \qquad\text{by the definition of $\gamma_h$ in \eqref{def: constant, gamma h}}
\end{align*}
and conclude the proof.
\end{proof}

Next, we recall a  bound for the tail cdf of geometric random variables.
Straightforward as it is, the bound is useful for our subsequent analysis.
For any $p \in (0,1)$, we say that $X$ is Geom$(p)$ if
\begin{align*}
    \P(X > k) = (1 - p)^k\qquad\forall k = 1,2,\ldots.
\end{align*}

\begin{lemma}[Lemma G.3 of \cite{wang2022eliminating}]
\label{lemma: tail bounds for geometric RVs}
    Let $a:(0,\infty) \to (0,\infty)$, $b:(0,\infty) \to (0,\infty)$ 
be two functions
such that 
$\lim_{\epsilon \downarrow 0} a(\epsilon) = 0, \lim_{\epsilon \downarrow 0} b(\epsilon) = 0$.
Let $\{U(\epsilon): \epsilon > 0\}$ be a family of geometric RVs with success rate $a(\epsilon)$,
i.e.
$\P(U(\epsilon) > k) = (1 - a(\epsilon))^{k}$ for $k \geq 1$.
For any $c > 1$, there exists $\epsilon_0 > 0$ such that
$$\exp\Big(-\frac{c\cdot a(\epsilon)}{b(\epsilon)}\Big) \leq \P\Big( U(\epsilon) > \frac{1}{b(\epsilon)} \Big) \leq \exp\Big(-\frac{a(\epsilon)}{c\cdot b(\epsilon)}\Big)
\ \ \ \forall \epsilon \in (0,\epsilon_0)
.$$
\end{lemma}

As an immediate consequence of Lemma~\ref{lemma: geometric expansion for communication},
the next result provides upper bounds for the information delay between a given set $S$ and any node $j$.
Note that $t_M = \bo(\log M)$.

\begin{lemma}\label{lemma: information delay}
Under Assumption~\ref{assumption: lower bound for h},
\begin{align}
    \inf_{ j \in [M] }
    \inf_{ S \subseteq [M]:\ S \neq \emptyset }
    \P\big(
       j \in  \bar S^{ t_M}
    \big) 
    \geq 
    \frac{\rho_h}{2 + \rho_h} \cdot \Big(1 - \exp(-\rho_h)\Big)
    -
    \lo(1)
    \qquad
    \text{as }M \to \infty,
    \nonumber
\end{align}
where 
\begin{align}
    t_M
    \delequal 
    \ceil{
    2\cdot 
    \frac{
        \log M - \log\big(2(1-\sqrt{\gamma_h})\big)
    }{
        \log( 1 + \rho_h )
    }
    }.
    \nonumber
\end{align}
\end{lemma}

\begin{proof}
Recall the definition of $\tilde t_M$ in \eqref{def: tilde t M, lemma: geometric expansion for communication}.
For any $M$ large enough, we have
$
\ceil{\tilde t_M} \leq t_M - 1.
$
Henceforth in the proof, we only consider such large $M$.
By Lemma~\ref{lemma: geometric expansion for communication},
\begin{align*}
    \P\big( |\bar S^{t_M - 1}| \geq M/2 \big)
    =
    \P\big( T_M \leq t_M - 1 \big)
    = 
    1 - \P\big( T_M > t_M - 1 \big)
    \geq 
    \frac{\rho_h}{2 + \rho_h}.
\end{align*}
Next, on event $\big\{ |\bar S^{t_M - 1}| \geq M/2  \big\}$,
given any $j \in [M]$
we either have $j \in \bar S^{t_M - 1}$ or $j \notin \bar S^{t_M - 1}$.
In the latter case, 
let
\begin{align*}
    V_j \delequal 
    \Big|
    \Big\{ i \in \bar S^{t_M - 1}:\ (i,j) \in E_{t_M}  \Big\}
    \Big|.
\end{align*}
First, if $V_j > 0$, we then have $j \in  S^{\bm h}_{t_M }$.
Furthermore, conditioned on event $\big\{ |\bar S^{t_M - 1}| \geq M/2  \big\}$,
the lower bound $c_h$ in Assumption~\ref{assumption: lower bound for h}
leads to the stochastic dominance relation that
\begin{align*}
    \text{Binomial}\bigg( \ceil{M/2} ,\ \frac{c^2_h}{\theta M}\bigg)
    \stleq 
    V_j\Big|\Big\{ |\bar S^{t_M - 1}| \geq M/2,\ j \notin \bar S^{t_M - 1} \Big\}.
\end{align*}
Therefore,
\begin{align*}
    & \inf_{j \in [M]}\P\big(j \in \bar S^{t_M} \ \big|\ |\bar S^{t_M - 1}| \geq M/2,\ j \notin \bar S^{t_M - 1}\big)
    \\
    & = \inf_{j \in [M]}\P\big(V_j > 0\ \big|\ |\bar S^{t_M - 1}| \geq M/2,\ j \notin \bar S^{t_M - 1}\big)
    \\ 
    & \geq 
    \P\Bigg( \text{Binomial}\bigg( \ceil{M/2} ,\ \frac{c^2_h}{\theta M}\bigg) > 0   \Bigg)
    =
    1 - \P\Bigg(
        \text{Geom}\bigg( \frac{c^2_h}{\theta M} \bigg) > \ceil{{M}/{2}}
    \Bigg)
    \\ 
    & = 
    1 - \exp\bigg( -\frac{c^2_h}{\theta M} \cdot \frac{M}{2}\bigg) - \lo(1)
    \qquad
    \text{(as $M \to \infty$) by Lemma~\ref{lemma: tail bounds for geometric RVs}}
    \\
    & = 1 - \exp(-\rho_h) - \lo(1).
\end{align*}
In summary,
\begin{align*}
    & \inf_{ j \in [M] }
    \P\big(
       j \in  \bar S^{t_M}
    \big) 
    \\
    &
    \geq 
    \inf_{ j \in [M] }
    \P\big(
       j \in  \bar S^{t_M},\ 
       |\bar S^{t_M - 1}| \geq M/2
    \big) 
    \\ 
    & \geq 
    \inf_{ j \in [M] }
    \P\big(
       j \in  \bar S^{t_M}\ \big|\ 
       |\bar S^{t_M - 1}| \geq M/2,\ j \notin  \bar S^{t_M - 1}
    \big) 
    \cdot 
    \P(|\bar S^{t_M - 1}| \geq M/2,\ j \notin  \bar S^{t_M - 1} )
    \\ 
    & + 
    \inf_{ j \in [M] }
    \P\big(
       j \in  \bar S^{t_M}\ \big|\ 
       |\bar S^{t_M - 1}| \geq M/2,\ j \in  \bar S^{t_M - 1}
    \big) 
    \cdot 
    \P(|\bar S^{t_M - 1}| \geq M/2, j \in  \bar S^{t_M - 1} )
    \\
     & = 
     \inf_{ j \in [M] }
    \P\big(
       j \in  \bar S^{t_M}\ \big|\ 
       |\bar S^{t_M - 1}| \geq M/2,\ j \notin  \bar S^{t_M - 1}
    \big) 
    \cdot 
    \P(|\bar S^{t_M - 1}| \geq M/2,\ j \notin  \bar S^{t_M - 1} )
    \\ 
    & \qquad + 
    \inf_{ j \in [M] } 1 \cdot 
    \P(|\bar S^{t_M - 1}| \geq M/2, j \in  \bar S^{t_M - 1} )
    \\ 
    & \geq 
    \inf_{ j \in [M] }
    \P\big(
       j \in  \bar S^{t_M}\ \big|\ 
       |\bar S^{t_M - 1}| \geq M/2,\ j \notin  \bar S^{t_M - 1}
    \big) 
    \\
    &\qquad
    \cdot 
    \Big(
    \P(|\bar S^{t_M - 1}| \geq M/2, j \notin  \bar S^{t_M - 1} )
    +
    \P(|\bar S^{t_M - 1}| \geq M/2, j \in  \bar S^{t_M - 1} )
    \Big)
    \\
    & = 
    \inf_{ j \in [M] }
    \P\big(
       j \in  \bar S^{t_M}\ \big|\ 
       |\bar S^{t_M - 1}| \geq M/2,\ j \notin  \bar S^{t_M - 1}
    \big) 
    \cdot 
    \P(|\bar S^{t_M - 1}| \geq M/2)
    \\ 
    & \geq 
    \Big( 1 - \exp(-\rho_h) - \lo(1) \Big)
    \cdot 
    \frac{\rho_h}{2 + \rho_h}.
\end{align*}
This concludes the proof.
\end{proof}

Now, we are ready to prove Lemma~\ref{lemma, information delay over sparse graphs}.

\begin{lemma*}[Lemma \ref{lemma, information delay over sparse graphs}]
Under Assumption~\ref{assumption: lower bound for h},
there exists some $\kappa \in (0,\infty)$ such that for any $\gamma > 0$, $M \geq 1$, and any non-emtpy $S \subseteq [M]$,
\begin{align}
    \P\big( j \notin \bar{S}^{ \gamma \cdot \kappa \cdot (\log M)^2 }\text{ for some }j \in [M] \big)
    \leq M^{-\gamma}.
    \label{claim, lemma, information delay over sparse graphs}
\end{align}
\end{lemma*}

\begin{proof}[Proof of Lemma~\ref{lemma, information delay over sparse graphs}]
Note that 
we have a trivial upper bound $1$ in the RHS of \eqref{claim, lemma, information delay over sparse graphs} under $M = 1$,
so the claim holds trivially for $M = 1$.
Henceforth in this proof we only consider $M \geq 2$.
Let
\begin{align}
    q_h & \delequal 1 - \frac{\rho_h}{4 + 2\rho_h}\cdot \big(1 - \exp(\rho_h)\big) \in (0,1),
    \nonumber
    \\
    t_M 
    & \delequal 
    \inf\bigg\{
        t \geq 1:\ 
        \inf_{ j \in [M] }
    \inf_{ S \subseteq [M]:\ S \neq \emptyset }
    \P(
       j \in  \bar S^{t}
    ) 
    \geq 
    1 - q_h
    \bigg\}.
        \label{proof, def t M, lemma, information delay over sparse graphs}
\end{align}
First of all, by Lemma~\ref{lemma: information delay},
for all $M$ large enough we have
\begin{align*}
    t_M \leq 
    \ceil{
    2\cdot 
    \frac{
        \log M - \log\big(2(1-\sqrt{\gamma_h})\big)
    }{
        \log( 1 + \rho_h )
    }
    },
\end{align*}
which confirms that 
$
t_M = \bo( \log M).
$
As a result, there exists $\tilde \kappa \in (0,\infty)$ such that $t_M \leq \tilde \kappa \log M$ for each $M \geq 2$.
Then, observe that by Markov property, it holds for any $k \geq 1$ that 
\begin{align}
     & \P\big( j \notin \bar{S}^{ k \cdot \tilde \kappa \log M }\text{ for some }j \in [M] \big) 
     \nonumber
     \\ 
     & \leq  M \cdot \sup_{j \in [M]} \P\big( j \notin \bar{S}^{ k \cdot  \tilde \kappa \log M } \big)
     \nonumber
     \\ 
     & \leq 
    M \cdot \prod_{ l = 1 }^k
    \sup_{j \in [M]} \P\big( j \notin \bar{S}^{ l \cdot  \tilde \kappa \log M }\ |\ j \notin \bar{S}^{ (l-1) \cdot  \tilde \kappa \log M } \big)
    \nonumber
    \\ 
    & \leq 
    M \cdot \bigg( 1 - \inf_{j \in [M]}\inf_{ S \subseteq [M]:\ S \neq \emptyset } 
    \P\big( j \in \bar{S}^{ \tilde \kappa \log M}\big)
    \bigg)^k
    \nonumber
    \\
     & \leq 
    M \cdot \bigg( 1 - \inf_{j \in [M]}\inf_{ S \subseteq [M]:\ S \neq \emptyset } 
    \P\big( j \in \bar{S}^{ t_M }\big)
    \bigg)^k
    \qquad\text{ due to }t_M \leq \tilde\kappa \log M
    \nonumber
    \\
    & \leq 
    M \cdot q_h^k
    \qquad\text{by the definition in \eqref{proof, def t M, lemma, information delay over sparse graphs}.}
    \label{proof, ineq, lemma, information delay over sparse graphs}
\end{align}
Lastly, given any $\gamma > 0$,
by setting 
\begin{align*}
    k = \frac{\gamma + 1}{\log(1/q_h)} \cdot \log M,
    \qquad
    \kappa =  \frac{\gamma + 1}{\log(1/q_h)} \cdot \tilde \kappa
\end{align*}
in the display \eqref{proof, ineq, lemma, information delay over sparse graphs},
we conclude the proof of claim~\eqref{claim, lemma, information delay over sparse graphs}.
\end{proof}

\subsection{Reward Information Delay}

We would like to add that the delay on spare graphs also leads to information asynchronization among clients, referred to as information delay. Such delays necessitate careful analysis in order to design an effective algorithm and address the challenges due to heavy-tailed observations and delayed information over sparse communication. To this end, we establish the following result regarding the delay in client information and prove that it is possible for the clients to stay synchronized within reasonable thresholds.

\begin{lemma}\label{lemma:information_delay}
Let $t_{m,j} = \max_{s \leq t} \{(m,j) \in E_s , c_m \neq c_j\}$ and $p = \frac{\eta}{TM}$. If $n_{m,i}(t) > 2\kappa (\log{M})^2\log{T}$ for any client $m$ and any arm $i$, then we obtain that with probability at least $1-p$, for $j \not\in S_0$,  
$\min_{m}n_{m,i}(t_{m,j}) \geq  \frac{1}{2}\min_{m}n_{m,i}(t)~\refstepcounter{equation}(\theequation)\label{lemma:A1-3}$, 
$\min_{m}n_{m,i}(t) \geq \frac{1}{2}n_{m,i}(t)~\refstepcounter{equation}(\theequation)\label{lemma:A1-2}$, and 
$N_{j,i}(t) = n_{m,i}(t_{m,j}) \geq  n_{m,i}(t) - \kappa (\log{M})^2\log{T}~\refstepcounter{equation}(\theequation)\label{lemma:A1-1}$. 
\end{lemma}

\subsection{Proof of Lemma~\ref{lemma:information_delay}}

\begin{lemma*}[Lemma \ref{lemma:information_delay}]
Let us assume that $p = \frac{\eta}{TM}$. Let us further assume that $L > 2\kappa (\log{M})^2\log{T}$ where $L$ is the length of the burn-in period.  Then we obtain with probability at least $1-p$, for $j \not\in S_0$, \begin{align}\label{lemma:A1-3}
    \min_{m}n_{m,i}(t_{m,j}) \geq  \frac{1}{2}\min_{m}n_{m,i}(t)
\end{align}
and 
\begin{align}\label{lemma:A1-2}
    \min_{m}n_{m,i}(t) \geq \frac{1}{2}n_{m,i}(t)
\end{align}
and 
\begin{align}\label{lemma:A1-1}
    N_{j,i}(t) = n_{m,i}(t_{m,j}) \geq  n_{m,i}(t) - \kappa (\log{M})^2\log{T}
\end{align}
\end{lemma*}

\begin{proof}[Proof of Lemma \ref{lemma:information_delay}]
    We demonstrate the proof steps as follows. 

    Based on Lemma in Section \ref{sec:graphs}, we obtain that 
     there exists $\kappa \in (0,\infty)$
such that given any $\gamma > 0$,
\begin{align*}
    \P\big( j \not\in \bar{S}^{\bm h}_{ \gamma \cdot \kappa \cdot (\log M)^2 }\text{ for some }j \in [M] \big)
    \leq 
    M^{-\gamma}.
\end{align*}

If we specify $\gamma = 2\log{T}$, we obtain that after $\kappa \cdot \log{T} \cdot (\log{M})^2$ steps, with probability at least $1 - M^{-2\log{T}}$, i.e. $1 - O(\frac{1}{T^2})$, all clients communicate. Equivalently, by the definition of $t_{m,j}$, we derive with probability $P_0=1 - O(\frac{1}{T^2})$,
\begin{align*}
    t - t_{m,j} \leq \kappa \cdot \log{T} \cdot (\log{M})^2. 
\end{align*}

Then we consider the above result for any client $j$ and any $t$, and derive that 
\begin{align}\label{eq:t.3}
    & P(\forall m, t, t - t_{m,j} \leq \kappa \cdot \log{T} \cdot (\log{M})^2) \notag \\ 
    & = 1 - P(\cup_{m,t}\{ t - t_{m,j} \leq \kappa \cdot \log{T} \cdot (\log{M})^2\}) \notag \\
    & \geq 1 - \sum_m\sum_tP(t - t_{m,j} \leq \kappa \cdot \log{T} \cdot (\log{M})^2) \notag \\
    & = 1 - MTP_0 = 1 - \frac{1}{T}
\end{align}
by the Bonferroni's inequality.

As a result, we obtain that with probability at least $1 - \frac{1}{T}$
\begin{align*}
    & n_{m,i}(t_{m,j}) \\
    & = n_{m,i}(t - (t-t_{m,j})) \\
    & \geq n_{m,i}(t - \kappa \cdot \log{T} \cdot (\log{M})^2)) \\
    & \geq n_{m,i}(t) - \kappa \cdot \log{T} \cdot (\log{M})^2)
\end{align*}
which concludes part of the statement, i.e. Equation (\ref{lemma:A1-1}).

Meanwhile, we note that the clients follow the information (possibly delayed) from the hub to decide on which arm to pull and clients in the hub use the same strategy all the time as they share the same non-delayed information. Notably, such delay is upper bounded by $\kappa \cdot \log{T} \cdot (\log{M})^2)$, which implies that for any $m$ 
\begin{align*}
\min_{m}n_{m,i}(t) \geq n_{m,i}(t) - \kappa \cdot \log{T} \cdot (\log{M})^2) 
\end{align*}

By the choice of $L$, we have that 
\begin{align*}
    & n_{m,i}(t) \geq n_{m,i}(L) \geq \frac{L}{K} \\
    & \geq 2\kappa \cdot \log{T} \cdot (\log{M})^2)
\end{align*}
which immediately implies that 
\begin{align*}
    & \min_{m}n_{m,i}(t) \geq n_{m,i}(t) - \kappa \cdot \log{T} \cdot (\log{M})^2) \\
    & \geq n_{m,i}(t) - \frac{1}{2}n_{m,i}(t) \\
    & = \frac{1}{2}n_{m,i}(t). 
\end{align*}

This completes the proof of Equation (\ref{lemma:A1-2}) in the statement. 

Additionally, we note that
\begin{align*}
   &  \min_{n}n_{m,i}(t_{m,j}) \\
   & \geq \frac{1}{2}n_{m,i}(t) \\
   & \geq \frac{1}{2}\min_{m}n_{m,i}(t)
\end{align*}
where the first inequality holds true using Equation (\ref{lemma:A1-2}) and the second inequality uses the fact that $n_{m,i}(t) \geq \min_{m}n_{m,i}(t)$. 

This completes the proof of Equation \eqref{lemma:A1-1} and thus the entire proof of the statement. 
\end{proof}

\section{Proof of Theorems}

\subsection{Proof of Theorem \ref{thm:1a2}}

\begin{proof}

First, we characterize the deviation of $\hat{\mu}_i^m(t)$ from the underlying groundtruth, the global mean value $\mu_i$, through mathematical induction, given the construction of $\hat{\mu}_{i}^{m}$.

We would like to highlight that we do not characterize the variance of the estimators, since we are in a heavy-tailed reward regime, where the estimators may not necessarily have finite variance. In fact, we show that the heavy-tailed reward estimator can directly be bounded by the heavy-tail dynamics using median-of-means. 

The next lemma is a two-sided version of Lemma 2 of \cite{bubeck2013bandits} and provides the concentration inequality for the median-of-means estimator.

\begin{lemma}[Lemma 2 of \cite{bubeck2013bandits}]
\label{lemma: concentration, median of means}
Let $\delta \in (0,1)$ and $\epsilon \in (0,1]$.
Let $X_n$'s be iid copies of $X$ with $\E X = \mu$,
and $\E|X - \mu|^{1 + \epsilon} \leq v$.
Let $k = \floor{ 8 \log( e^{1/8}/\delta )\wedge n/2 }$ and $N = \floor{n/k}$.
Let
\begin{align*}
    \hat \mu_j = \frac{1}{N}\sum_{ t = (j-1)N + 1  }^{jN}
    X_t
    \qquad\forall j = 1,2,\ldots,k,
\end{align*}
and let $\widehat{\mu}_M$ be the median of $(\hat \mu_j)_{j = 1,2,\ldots,k}$.
Then, with probability at least $1 - 2\delta$,
\begin{align*}
    |\widehat{\mu}_M - \mu| \leq (12v)^{ \frac{1}{1 + \epsilon} } 
        \bigg(
            \frac{
                16 \log(e^{1/8}\delta^{-1})
            }{n}
        \bigg)^{ \frac{\epsilon}{1 + \epsilon}  }.
\end{align*}
\end{lemma}

By using Lemma \ref{lemma: concentration, median of means} with respect to the rewards from clients, denoted by set $rw_t = \{s \leq t: \{r_i^j(s)\}_{j \in \mathcal{N}_{m,t}(s), a_j^s = i}\}$ where $m$ is a hub, we derive that the global estimator at the hub $m \in S_0$ meets the following. 

Formally, we have that \begin{align*}
    | \hat{\mu}_i^m(t)  - \mu^i | \leq 2C\rho^{ \frac{1}{1+\epsilon} }\bigg(\frac{c \log(1/\delta) }{|rw_t|}\bigg)^{ \frac{\epsilon}{1 + \epsilon} }
    \ \
    \text{with probability at least }1 - \delta
    \qquad \forall n \geq 1.
\end{align*}
where $\hat{\mu}_i^m(t)$ is the median of the means constructed as illustrated in Lemma \ref{lemma: concentration, median of means} based on Algorithm \hl{4}.

It is worth noting that by definition, the size of $rw_t$, which we denote as $|rw_t|$, is equivalent to $\sum_{j \in S_0}n_{j,i}(t) \geq |S_0|\min_{j}n_{j,i}(t)$.  

Meanwhile, by our result on information delay as established in Section \ref{sec:graphs-sub2}, we obtain that when $L \geq 2\kappa (\log{M})^2 \log{T}$, 
\begin{align*}
    \min_{m}n_{m,i}(t) \geq \frac{1}{2}n_{m,i}(t)
\end{align*}
which immediately implies that 
\begin{align*}
   & |rw_t| = \sum_{j \in S_0}n_{j,i}(t) \\
   & \geq |S_0|\min_{j}n_{j,i}(t) \\
   & \geq \frac{|S_0|}{2}n_{m,i}(t)
\end{align*}

Subsequently, we derive that the following concentration inequality 
\begin{align*}
    & | \hat{\mu}_i^m(t)  - \mu^i | \\
    & \leq 2C\rho^{ \frac{1}{1+\epsilon} }\bigg(\frac{c \log(1/\delta) }{|rw_t|}\bigg)^{ \frac{\epsilon}{1 + \epsilon} } \\
    & \leq 2C\rho^{ \frac{1}{1+\epsilon} }\bigg(\frac{2c \log(1/\delta) }{|S_0| \cdot n_{m,i}(t)}\bigg)^{ \frac{\epsilon}{1 + \epsilon}}  \text{with probability at least }1 - \delta
    \qquad \forall n \geq 1 .
\end{align*}

Then we consider the following regret decomposition that allows us to leverage the result from the above concentration inequality.

We recall that the optimal arm is 
\begin{align*}
    i^* = \arg\max_{i}\mu^{m}_i = \arg\max_{i}\mu_i. 
\end{align*}

By decomposing $R_T$, we obtain that  
\begin{align*}
    R_T & =   \frac{1}{M}(\max_i\sum_{t=1}^T\sum_{m=1}^M\mu^{m}_i - \sum_{t=1}^T\sum_{m=1}^M\mu^{m}_{a_t^m}) \\
    & = \sum_{t=1}^T\frac{1}{M}\sum_{m=1}^M\mu^{m}_{i^*} - \sum_{t=1}^T\frac{1}{M}\sum_{m=1}^M\mu^{m}_{a_t^m} \\
    & \leq \sum_{t = 1}^{L}|\frac{1}{M}\sum_{m=1}^M\mu^{m}_{i^*} - \frac{1}{M}\sum_{m=1}^M\mu^{m}_{a_t^m}|+ \sum_{t = L + 1}^T(\frac{1}{M}\sum_{m=1}^M\mu^{m}_{i^*} - \frac{1}{M}\sum_{m=1}^M\mu^{m}_{a_t^m}) \\
    & \leq L + \sum_{t = L + 1}^T(\frac{1}{M}\sum_{m=1}^M\mu^{m}_{i^*} - \frac{1}{M}\sum_{m=1}^M\mu^{m}_{a_t^m}) \\
    & = L + \sum_{t = L + 1}^T(\mu_{i^*} - \frac{1}{M}\sum_{m=1}^M\mu^{m}_{a_t^m}) \\
    & = L+ ((T - L) \cdot \mu_{i^*} - \frac{1}{M}\sum_{m=1}^M\sum_{i = 1}^Kn_{m,i}(T)\mu^m_i)
\end{align*}
where the first inequality uses that $|a| \geq a$ for any $a$ and the second inequality holds by noting that $0 < \mu_{i}^j < 1$ 

We consider the number of pulls of arms resulting from the UCB strategies as follows. 

We assert that the factors causing the selection of a sub-optimal arm $i$ are explicitly defined by the decision rule of Algorithm~\ref{alg:dr}. Specifically, the outcome $a_t^m = i$ occurs when any of the following conditions is satisfied:
\begin{itemize}
    \item Case 1: $\Tilde{\mu}^{m}_i - \mu_i > C\rho^{ \frac{1}{1+\epsilon} }\bigg(\frac{2c \log(1/\delta) }{|S_0| \cdot n_{m,i}(t)}\bigg)^{ \frac{\epsilon}{1 + \epsilon}},$\;
    \item Case 2: $- \Tilde{\mu}^{m}_{i^*} + \mu_{i^*} > C\rho^{ \frac{1}{1+\epsilon} }\bigg(\frac{2c \log(1/\delta) }{|S_0| \cdot n_{m,i^*}(t)}\bigg)^{ \frac{\epsilon}{1 + \epsilon}}$,\;
    \item Case 3: $\mu_{i^*} - \mu_i < 2C\rho^{ \frac{1}{1+\epsilon} }\bigg(\frac{2c \log(1/\delta) }{|S_0| \cdot n_{m,i}(t)}\bigg)^{ \frac{\epsilon}{1 + \epsilon}}$.
\end{itemize}

Subsequently, we obtain that the number of arm pulls of arm $i$ for client $m$ can be upper bounded as follows:
\begin{align*}
    n_{m,i}(T) & \leq l + \sum_{t=L+1}^T1_{\{a_t^m = i, n_{m,i}(t) > l\}}  \\
    & \leq l + \sum_{t=L+1}^T1_{\{\Tilde{\mu}^m_i - C\rho^{ \frac{1}{1+\epsilon} }\bigg(\frac{2c \log(1/\delta) }{|S_0| \cdot n_{m,i}(t)}\bigg)^{ \frac{\epsilon}{1 + \epsilon}} > \mu_i, n_{m,i}(t-1) \geq l\}} \\
    & \qquad \qquad + \sum_{t=L+1}^T1_{\{\Tilde{\mu}^m_{i^*} + C\rho^{ \frac{1}{1+\epsilon} }\bigg(\frac{2c \log(1/\delta) }{|S_0| \cdot n_{m,i^*}(t)}\bigg)^{ \frac{\epsilon}{1 + \epsilon}} < \mu_{i^*}, n_{m,i}(t-1) \geq l\}} \\ 
    & \qquad \qquad + \sum_{t=L+1}^T1_{\{\mu_i + 2C\rho^{ \frac{1}{1+\epsilon} }\bigg(\frac{2c \log(1/\delta) }{|S_0| \cdot n_{m,i}(t)}\bigg)^{ \frac{\epsilon}{1 + \epsilon}} > \mu_{i^*},n_{m,i}(t-1) \geq l\}}. 
\end{align*}

By taking expected values over $n_{m,i}(t)$ conditional on $A_{\zeta, \delta}$, we derive 
\begin{align}\label{eq:en}
    & E[n_{m,i}(T) | A_{\zeta, \delta}] \notag \\
    & = l + \sum_{t=L+1}^TP(\Tilde{\mu}^m_i - C\rho^{ \frac{1}{1+\epsilon} }\bigg(\frac{2c \log(1/\delta) }{|S_0| \cdot n_{m,i}(t)}\bigg)^{ \frac{\epsilon}{1 + \epsilon}} > \mu_i, n_{m,i}(t-1) \geq l| A_{\zeta, \delta}) \notag \\
    & \qquad \qquad + \sum_{t=L+1}^TP(\Tilde{\mu}^m_{i^*} + C\rho^{ \frac{1}{1+\epsilon} }\bigg(\frac{2c \log(1/\delta) }{|S_0| \cdot n_{m,i}(t)}\bigg)^{ \frac{\epsilon}{1 + \epsilon}} < \mu_{i^*}, n_{m,i}(t-1) \geq l | A_{\zeta, \delta}) \notag \\
    & \qquad \qquad + \sum_{t=L+1}^TP(\mu_i + 2C\rho^{ \frac{1}{1+\epsilon} }\bigg(\frac{2c \log(1/\delta) }{|S_0| \cdot n_{m,i}(t)}\bigg)^{ \frac{\epsilon}{1 + \epsilon}} > \mu_{i^*},n_{m,i}(t-1) \geq l | A_{\zeta, \delta}) \notag \\
    & = l + \sum_{t=L+1}^TP(Case 1, n_{m,i}(t-1) \geq l | A_{\zeta, \delta}) \\
    & \qquad \qquad + \sum_{t=L+1}^TP(Case 2, n_{m,i}(t-1) \geq l | A_{\zeta, \delta}) + \sum_{t=L+1}^TP(Case 3, n_{m,i}(t-1) \geq l | A_{\zeta, \delta})
\end{align}
where $l =  \frac{2c\log{T}}{|S_0|(\frac{\Delta_i}{2C\rho^{\frac{1}{1+\epsilon}}})^{\frac{1+\epsilon}{\epsilon}}}$ with $\Delta_i = \mu_{i^*} - \mu_i$..

For the last term in (\ref{eq:en}), we have 
\begin{align}\label{eq:case4}
    \sum_{t= L+1}^TP(Case4: \mu_i + 2C\rho^{ \frac{1}{1+\epsilon} }\bigg(\frac{2c \log(1/\delta) }{|S_0| \cdot n_{m,i}(t)}\bigg)^{ \frac{\epsilon}{1 + \epsilon}}  > \mu_{i^*},n_{m,i}(t-1) \geq l | A_{\zeta, \delta})) = 0
\end{align}
since the choice of $l$ satisfies $l \geq \frac{2c\log{T}}{|S_0|(\frac{\Delta_i}{2C\rho^{\frac{1}{1+\epsilon}}})^{\frac{1+\epsilon}{\epsilon}}}$ with $\Delta_i = \mu_{i^*} - \mu_i$.

We start with the two terms, and subsequently obtain that on event $A_{\zeta, \delta}$
\begin{align}\label{eq:cases}
    & \sum_{t= L + 1}^TP(Case 2, n_{m,i}(t-1) \geq l | A_{\zeta, \delta}) + \sum_{t=1}^T P(Case 3, n_{m,i}(t-1) \geq l | A_{\zeta, \delta}) \notag \\
    & \leq \sum_{t= L + 1}^TP(\Tilde{\mu}_{m,i} - \mu_i > C\rho^{ \frac{1}{1+\epsilon} }\bigg(\frac{2c \log(1/\delta) }{|S_0| \cdot n_{m,i}(t)}\bigg)^{ \frac{\epsilon}{1 + \epsilon}} | A_{\zeta, \delta}) + \\
    & \qquad \sum_{t=1}^T P( - \Tilde{\mu}_{m,i^*} + \mu_{i^*} > C\rho^{ \frac{1}{1+\epsilon} }\bigg(\frac{2c \log(1/\delta) }{|S_0| \cdot n_{m,i}(t)}\bigg)^{ \frac{\epsilon}{1 + \epsilon}} | A_{\zeta, \delta}) \notag \\
    & \leq \sum_{t=1}^T(\frac{1}{t^2} ) + \sum_{t=1}^T(\frac{1}{t^2}) \leq \frac{\pi^2}{3}
\end{align}
where the first inequality utilizes the property of the probability measure when removing the event $n_{m,i}(t-1) \geq l$ and the second inequality holds by the aforementioned concentration inequality based on median-of-means. 

As a result, by the above decomposition, we derive that 
\begin{align*}
    & E[n_{m,i}(t)|A_{\zeta, \delta}] \\
    & \leq l + \sum_{t=L+1}^TP(Case 1, n_{m,i}(t-1) \geq l | A_{\zeta, \delta}) \\
    & \qquad \qquad + \sum_{t=L+1}^TP(Case 2, n_{m,i}(t-1) \geq l | A_{\zeta, \delta}) + \sum_{t=L+1}^TP(Case 3, n_{m,i}(t-1) \geq l | A_{\zeta, \delta}) \\
    & \leq \frac{2c\log{T}}{|S_0|(\frac{\Delta_i}{2C\rho^{\frac{1}{1+\epsilon}}})^{\frac{1+\epsilon}{\epsilon}}} + \frac{\pi^2}{3}
\end{align*}

Consequently, we consider the following upper bound on $R_T$ by the previous decomposition, which gives us that 
\begin{align*}
  & R_T \leq L+ ((T - L) \cdot \mu_{i^*} - \frac{1}{M}\sum_{m=1}^M\sum_{i = 1}^Kn_{m,i}(T)\mu^m_i) \\
  & \leq 2\kappa (\log{M})^2 \log{T} +  \sum_{m}\sum_{i}n_{m,i}(t)\Delta_i
\end{align*}
and 
\begin{align*}
   & E[R_T| A_{\zeta, \delta}] \\
   & \leq L + (\sum_{m}(\frac{2c\log{T}}{|S_0|(\frac{\Delta_i}{2C\rho^{\frac{1}{1+\epsilon}}})^{\frac{1+\epsilon}{\epsilon}}} + \frac{\pi^2}{3})) \cdot \Delta_i \\
   & \leq L + M(\frac{2c\Delta_i\log{T}}{|S_0|(\frac{\Delta_i}{2C\rho^{\frac{1}{1+\epsilon}}})^{\frac{1+\epsilon}{\epsilon}}} + \frac{\pi^2}{3} \Delta_i) = O(\frac{M}{|S_0|}\log{T})
\end{align*}

Meanwhile, by the definition of $A_{\zeta, \delta}$, we obtain that 
 $S_0 \geq M^{2-\alpha-\zeta}$. This completes the first part of the statement. 

 Then by direct computation with $S_0 \geq M^{2-\alpha-\zeta}$, we derive the upper bound on $R_T$ with respect to $M$ and $T$, which is 
 \begin{align*}
     E[R_T|A_{\zeta, \delta}] \leq O(\frac{M}{|S_0|}\log{T}) \leq O(M^{\alpha-1+\zeta}\log{T})
 \end{align*}
 which concludes the second part of the statement and thus completes proof of Theorem \ref{thm:1a2}. 
\end{proof}


\subsection{Proof of Theorem \ref{thm:1a}}

\begin{proof}
We would like to highlight that the proof of Theorem \ref{thm:1a} overlaps with the proof of Theorem \ref{thm:1a2}, since they both consider Algorithm \ref{alg:heter-L} in a homogeneous setting. For completeness, we present the full proof here, and may repeat some steps in the previous proof. 

First, we establish the statistical property of the estimator $\hat{\mu}_{i}^m(t)$. 
By using Lemma \ref{lemma: concentration, median of means} with respect to the rewards from clients, denoted by set $rw_t = \{s \leq t: \{r_i^j(s)\}_{j \in \mathcal{N}_{m,t}(s), a_j^s = i}\}$ where $m$ is a hub, we derive that the global estimator at the hub $m \in S_0$ meets the following. 

Formally, we have that \begin{align*}
    | \hat{\mu}_i^m(t)  - \mu^i | \leq (12v)^{ \frac{1}{1 + \epsilon} } 
        \bigg(
            \frac{
                16 \log(e^{1/8}\delta^{-1})
            }{|rw_{t}|}
        \bigg)^{ \frac{\epsilon}{1 + \epsilon}  }
    \ \
    \text{with probability at least }1 - \delta
    \qquad \forall n \geq 1,
\end{align*}
where $\hat{\mu}_i^m(t)$ is the median of the means constructed as illustrated in Lemma \ref{lemma: concentration, median of means}.

It is worth noting that by definition, the size of $rw_t^i$, which we denote as $|rw_t^i|$, is the total number of arm pulls of arm $i$ by time $t$ with respect to all clients in the hub. 

Meanwhile, we have that the total number of the hub size is the total number of arm pulls of all arms with respect to all clients in the club.

Given the modifications to the algorithm where we 
add the burn-in period ($L \geq 2\kappa K \log{M}\log{T}$), we have the following claim. We consider the following modification to the information transmission. When $S_0^t < a$ (possibly connecting $M$ with $T$ by considering large-scale systems), the transmission between the hub and the non-hub is paused, which implies that $rw_t^i \leq rw_{t-1}^i + 1$. Consider $\tau(t) = \max{\{s \leq t: S_0^s \geq a\}}$ where $a = M^{\frac{1}{\alpha} - \zeta}$. Equivalently, we have that $rw_{t}^i \geq a \cdot n_{m,i}(\tau)$. Meanwhile, it is worth noting that the difference between $t$ and $\tau$ can be upper bounded by $\kappa \log{M}\log{T}$, noting that $|S_0^1|, |S_0^2|, \ldots, |S_0^t|$ are i.i.d random variables, and thus the condition of $\tau$ follows a geometric distribution with an exponential decay based on Lemma \ref{lemma, hub: stochastic lower bound} in Section \ref{sec:graphs-sub2}, with probability $1-\frac{1}{MT}$ conditional on event $A_{\alpha, \epsilon}$. Consequently, we derive that when $t > L$, on event $A_{\alpha, \epsilon}$ with probability  $1-\frac{1}{MT}$
\begin{align*}
    rw_{t}^i & \geq a \cdot n_{m,i}(\tau) \\
    & \geq a \cdot (n_{m,i}(t) - \kappa\log{M}\log{T} )\\
    & \geq a \cdot \frac{1}{2}n_{m,i}(t) \\
    & \doteq |S_0|\cdot \frac{1}{2}n_{m,i}(t)
\end{align*}
by noting that we have $L \geq  2\kappa K\log{M}\log{T}$, and as such $n_{m,i}(t) \geq 2\kappa\log{M}\log{T}$ for any arm $i$.

We define event $A_{\zeta, \delta} = A_{\zeta, \delta} \cap A_{\alpha, \epsilon}$. 

Subsequently, we derive that 
\begin{align*}
    & P(A_{\zeta, \delta}) \\
    & \geq 1- (1-P(A_{\zeta, \delta}) + 1 - P(A_{\alpha, \epsilon})) \\
    & \geq 1 - \frac{1}{M^2} - \frac{\eta}{TM}
\end{align*}


Subsequently, we derive that the following concentration inequality 
\begin{align*}
    & | \hat{\mu}_i^m(t)  - \mu^i | \\
    & \leq 2C\rho^{ \frac{1}{1+\epsilon} }\bigg(\frac{c \log(1/\delta) }{|rw_t|}\bigg)^{ \frac{\epsilon}{1 + \epsilon} } \\
    & \leq 2C\rho^{ \frac{1}{1+\epsilon} }\bigg(\frac{2c \log(1/\delta) }{|S_0| \cdot n_{m,i}(t)}\bigg)^{ \frac{\epsilon}{1 + \epsilon}}  \text{with probability at least }1 - \delta
    \qquad \forall n \geq 1 .
\end{align*}

We assert that the factors causing the selection of a sub-optimal arm $i$ are explicitly defined by the decision rule of Algorithm~\ref{alg:dr}. Specifically, the outcome $a_t^m = i$ occurs when any of the following conditions is satisfied:
\begin{itemize}
    \item Case 1: $\Tilde{\mu}^{m}_i - \mu_i > C\rho^{ \frac{1}{1+\epsilon} }\bigg(\frac{2c \log(1/\delta) }{|S_0| \cdot n_{m,i}(t)}\bigg)^{ \frac{\epsilon}{1 + \epsilon}},$\;
    \item Case 2: $- \Tilde{\mu}^{m}_{i^*} + \mu_{i^*} > C\rho^{ \frac{1}{1+\epsilon} }\bigg(\frac{2c \log(1/\delta) }{|S_0| \cdot n_{m,i^*}(t)}\bigg)^{ \frac{\epsilon}{1 + \epsilon}}$,\;
    \item Case 3: $\mu_{i^*} - \mu_i < 2C\rho^{ \frac{1}{1+\epsilon} }\bigg(\frac{2c \log(1/\delta) }{|S_0| \cdot n_{m,i}(t)}\bigg)^{ \frac{\epsilon}{1 + \epsilon}}$.
\end{itemize}

Subsequently, we obtain that the number of arm pulls of arm $i$ for client $m$ can be upper bounded as follows:
\begin{align*}
    n_{m,i}(T) & \leq l + \sum_{t=L+1}^T1_{\{a_t^m = i, n_{m,i}(t) > l\}}  \\
    & \leq l + \sum_{t=L+1}^T1_{\{\Tilde{\mu}^m_i - C\rho^{ \frac{1}{1+\epsilon} }\bigg(\frac{2c \log(1/\delta) }{|S_0| \cdot n_{m,i}(t)}\bigg)^{ \frac{\epsilon}{1 + \epsilon}} > \mu_i, n_{m,i}(t-1) \geq l\}} \\
    & \qquad \qquad + \sum_{t=L+1}^T1_{\{\Tilde{\mu}^m_{i^*} + C\rho^{ \frac{1}{1+\epsilon} }\bigg(\frac{2c \log(1/\delta) }{|S_0| \cdot n_{m,i^*}(t)}\bigg)^{ \frac{\epsilon}{1 + \epsilon}} < \mu_{i^*}, n_{m,i}(t-1) \geq l\}} \\ 
    & \qquad \qquad + \sum_{t=L+1}^T1_{\{\mu_i + 2C\rho^{ \frac{1}{1+\epsilon} }\bigg(\frac{2c \log(1/\delta) }{|S_0| \cdot n_{m,i}(t)}\bigg)^{ \frac{\epsilon}{1 + \epsilon}} > \mu_{i^*},n_{m,i}(t-1) \geq l\}}. 
\end{align*}

By taking expected values over $n_{m,i}(t)$ conditional on $A_{\zeta, \delta}$, we derive 
\begin{align}\label{eq:en}
    & E[n_{m,i}(T) | A_{\zeta, \delta}] \notag \\
    & = l + \sum_{t=L+1}^TP(\Tilde{\mu}^m_i - C\rho^{ \frac{1}{1+\epsilon} }\bigg(\frac{2c \log(1/\delta) }{|S_0| \cdot n_{m,i}(t)}\bigg)^{ \frac{\epsilon}{1 + \epsilon}} > \mu_i, n_{m,i}(t-1) \geq l| A_{\zeta, \delta}) \notag \\
    & \qquad \qquad + \sum_{t=L+1}^TP(\Tilde{\mu}^m_{i^*} + C\rho^{ \frac{1}{1+\epsilon} }\bigg(\frac{2c \log(1/\delta) }{|S_0| \cdot n_{m,i}(t)}\bigg)^{ \frac{\epsilon}{1 + \epsilon}} < \mu_{i^*}, n_{m,i}(t-1) \geq l | A_{\zeta, \delta}) \notag \\
    & \qquad \qquad + \sum_{t=L+1}^TP(\mu_i + 2C\rho^{ \frac{1}{1+\epsilon} }\bigg(\frac{2c \log(1/\delta) }{|S_0| \cdot n_{m,i}(t)}\bigg)^{ \frac{\epsilon}{1 + \epsilon}} > \mu_{i^*},n_{m,i}(t-1) \geq l | A_{\zeta, \delta}) \notag \\
    & = l + \sum_{t=L+1}^TP(Case 1, n_{m,i}(t-1) \geq l | A_{\zeta, \delta}) \\
    & \qquad \qquad + \sum_{t=L+1}^TP(Case 2, n_{m,i}(t-1) \geq l | A_{\zeta, \delta}) + \sum_{t=L+1}^TP(Case 3, n_{m,i}(t-1) \geq l | A_{\zeta, \delta})
\end{align}
where $l =  $ with $\Delta_i = \mu_{i^*} - \mu_i$..

For the last term in (\ref{eq:en}), we have 
\begin{align}\label{eq:case4}
    \sum_{t= L+1}^TP(Case3: \mu_i + 2C\rho^{ \frac{1}{1+\epsilon} }\bigg(\frac{2c \log(1/\delta) }{|S_0| \cdot n_{m,i}(t)}\bigg)^{ \frac{\epsilon}{1 + \epsilon}}  > \mu_{i^*},n_{m,i}(t-1) \geq l | A_{\zeta, \delta})) = 0
\end{align}
which also implies that 
\begin{align}\label{eq:case4}
    \sum_{t= L+1}^TP(Case4: 2C\rho^{ \frac{1}{1+\epsilon} }\bigg(\frac{2c \log(1/\delta) }{|S_0| \cdot n_{m,i}(t)}\bigg)^{ \frac{\epsilon}{1 + \epsilon}}  > \Delta_i,n_{m,i}(t-1) \geq l | A_{\zeta, \delta})) = 0 
\end{align}
since the choice of $l$ satisfies $l \geq \frac{2cN\log{T}}{(\frac{\Delta_i}{2\rho^{\frac{1}{1+\epsilon}}})^{\frac{1+\epsilon}{\epsilon}}}$ with the choice of $|S_0| = M^{\frac{1}{\alpha } - \zeta}$. 

Next, we consider the first two terms, and as in the proof of Theorem \ref{thm:1a2}, we subsequently obtain that on event $A_{\zeta, \delta}$
\begin{align}\label{eq:cases}
    & \sum_{t= L + 1}^TP(Case 2, n_{m,i}(t-1) \geq l | A_{\zeta, \delta}) + \sum_{t=1}^T P(Case 3, n_{m,i}(t-1) \geq l | A_{\zeta, \delta}) \notag \\
    & \leq \sum_{t= L + 1}^TP(\Tilde{\mu}_{m,i} - \mu_i > C\rho^{ \frac{1}{1+\epsilon} }\bigg(\frac{2c \log(1/\delta) }{|S_0| \cdot n_{m,i}(t)}\bigg)^{ \frac{\epsilon}{1 + \epsilon}} | A_{\zeta, \delta}) + \\
    & \qquad \sum_{t=1}^T P( - \Tilde{\mu}_{m,i^*} + \mu_{i^*} > C\rho^{ \frac{1}{1+\epsilon} }\bigg(\frac{2c \log(1/\delta) }{|S_0| \cdot n_{m,i}(t)}\bigg)^{ \frac{\epsilon}{1 + \epsilon}} | A_{\zeta, \delta}) \notag \\
    & \leq \sum_{t=1}^T(\frac{1}{t^2} ) + \sum_{t=1}^T(\frac{1}{t^2}) \leq \frac{\pi^2}{3}
\end{align}
where the first inequality utilizes the property of the probability measure when removing the event $n_{m,i}(t-1) \geq l$ and the second inequality holds by the aforementioned concentration inequality based on median-of-means. 

As a result, by the above decomposition, we derive that 
\begin{align*}
    & E[n_{m,i}(t)|A_{\zeta, \delta}] \\
    & \leq l + \sum_{t=L+1}^TP(Case 1, n_{m,i}(t-1) \geq l | A_{\zeta, \delta}) \\
    & \qquad \qquad + \sum_{t=L+1}^TP(Case 2, n_{m,i}(t-1) \geq l | A_{\zeta, \delta}) + \sum_{t=L+1}^TP(Case 3, n_{m,i}(t-1) \geq l | A_{\zeta, \delta}) \\
    & \leq \frac{2c\log{T}}{|S_0|(\frac{\Delta_i}{2C\rho^{\frac{1}{1+\epsilon}}})^{\frac{1+\epsilon}{\epsilon}}} + \frac{\pi^2}{3}
\end{align*}

Consequently, we consider the following upper bound on $R_T$ by the previous decomposition, which gives us that 
\begin{align*}
  & R_T \leq L+ ((T - L) \cdot \mu_{i^*} - \frac{1}{M}\sum_{m=1}^M\sum_{i = 1}^Kn_{m,i}(T)\mu^m_i) \\
  & \leq 2\kappa (\log{M})^2 \log{T} +  \sum_{m}\sum_{i}n_{m,i}(t)\Delta_i
\end{align*}
and 
\begin{align*}
   & E[R_T| A_{\epsilon, \delta, \zeta}] \\
   & \leq L + (\sum_{m}(\frac{2c\log{T}}{|S_0|(\frac{\Delta_i}{2C\rho^{\frac{1}{1+\epsilon}}})^{\frac{1+\epsilon}{\epsilon}}} + \frac{\pi^2}{3})) \cdot \Delta_i \\
   & \leq L + M(\frac{2c\Delta_i\log{T}}{|S_0|(\frac{\Delta_i}{2C\rho^{\frac{1}{1+\epsilon}}})^{\frac{1+\epsilon}{\epsilon}}} + \frac{\pi^2}{3} \Delta_i) = O(\frac{M}{|S_0|}\log{T})
\end{align*}
where $L \geq 2\kappa K (\log{M})^2 \log{T}$. This completes the first part of the proof. 

Next we consider the value of $|S_0|$, by the definition of $a$, we obtain that 
 $S_0 \geq M^{\frac{1}{\alpha }-\zeta}$.  Then by direct computation with $S_0 \geq M^{\frac{1}{\alpha }-\zeta}$, we derive the upper bound on $R_T$ with respect to $M$ and $T$, which is 
 \begin{align*}
     E[R_T|A_{\epsilon, \delta, \zeta}] \leq O(\frac{M}{|S_0|}\log{T}) \leq O(M^{1 -\frac{1}{\alpha }+\zeta}\log{T})
 \end{align*}
 which concludes the second part of the proof and thus completes proof of Theorem \ref{thm:1a}.

\end{proof}

\begin{corollary}[Full information setting]\label{corollary:homo}
Let us assume the assumptions in Theorem \ref{thm:1a2} hold, except that we are in a full information setting instead of a partial feedback (bandit) setting. This means that the rewards of all arms are observable at each time step, rather than only the reward of the pulled arm. Under this setting, the same regret bound holds.  
\end{corollary}

\subsection{Proof of Corollary \ref{corollary:homo}}

\begin{proof}
     Let us consider a full information setting where the clients observe the rewards of all arms. 
     
     It is equivalent to implying that $\sum_{t=1}^T|S_0^t| = rw_t^i$, by the fact that every time if there is a client in $S_0^t$, then there is a new sample about arm $i$ that is shared within the hub. 
     
     Subsequently, the entire analysis in the proof of Theorem \ref{thm:1a2} comes through, which concludes the proof. 
\end{proof}

\subsection{Proof of Theorem \ref{thm:heter}}

\begin{proof}
We would like to emphasize that the graph property we establish for heavy-tailed graphs does not depend on the rewards, and as such, the statements in Section \ref{sec:graphs} holds. This also implies that, compared to the homogeneous case, the key difference is in the reward aggregation, which we demonstrate in the following. 

Again, based on Section \ref{sec:graphs-sub2}, we have the following result regarding the information delay. 

\begin{proposition}
Let us assume that $p = \frac{\eta}{TM}$. Let us further assume that $L > 2\kappa (\log{M})^2\log{T}$ where $L$ is the length of the burn-in period.  Then we obtain with probability at least $1-p$,for $j \not\in S_0$, \begin{align*}
    \min_{m}n_{m,i}(h^t_{m,j}) \geq  \frac{1}{2}\min_{m}n_{m,i}(t)
\end{align*}
and 
\begin{align*}
    \min_{m}n_{m,i}(t) \geq \frac{1}{2}n_{m,i}(t)
\end{align*}
and 
\begin{align*}
    N_{j,i}(t) = n_{m,i}(h^t_{m,j}) \geq  n_{m,i}(t) - \kappa (\log{M})^2\log{T}. 
\end{align*}
\end{proposition}

With this, we proceed to establish the statistical property of the global estimator $\tilde{\mu}_i^m(t)$.

In the heterogeneous setting, the reward distribution is the same as in the homogeneous setting, which, however, differs in clients even for the same arm. That being said, the following lemma still holds
\begin{lemma}[Lemma 2 of \cite{bubeck2013bandits}]
\label{lemma: concentration, median of means}
Let $\delta \in (0,1)$ and $\epsilon \in (0,1]$.
Let $X_n$'s be iid copies of $X$ with $\E X = \mu$,
and $\E|X - \mu|^{1 + \epsilon} \leq v$.
Let $k = \floor{ 8 \log( e^{1/8}/\delta )\wedge n/2 }$ and $N = \floor{n/k}$.
Let
\begin{align*}
    \hat \mu_j = \frac{1}{N}\sum_{ t = (j-1)N + 1  }^{jN}
    X_t
    \qquad\forall j = 1,2,\ldots,k,
\end{align*}
and let $\widehat{\mu}_M$ be the median of $(\hat \mu_j)_{j = 1,2,\ldots,k}$.
Then, with probability at least $1 - 2\delta$,
\begin{align*}
    |\widehat{\mu}_M - \mu| \leq (12v)^{ \frac{1}{1 + \epsilon} } 
        \bigg(
            \frac{
                16 \log(e^{1/8}\delta^{-1})
            }{n}
        \bigg)^{ \frac{\epsilon}{1 + \epsilon}  }.
\end{align*}
\end{lemma}

It is worth noting that by using the above lemma, we have the following concentration inequality for the local estimator $\hat{\mu}$ at each client, with $n_{m,i}(t)$ samples, 
\begin{align*}
    | \hat{\mu}_m^i(n_{m,i}(t);k)  - \mu_m^i | \leq (12v)^{ \frac{1}{1 + \epsilon} } 
        \bigg(
            \frac{
                16 \log(e^{1/8}\delta^{-1})
            }{n_{m,i}(t)}
        \bigg)^{ \frac{\epsilon}{1 + \epsilon}  }  
    \ \
    \text{with probability at least }1 - \delta
    \qquad \forall n \geq 1.
\end{align*}

Now we proceed to establish the concentration inequality of the global estimator $\tilde{\mu}_{i}^m(t)$, which is constructed by 
\begin{align*}
  & \Tilde{\mu}^m_i(t+1) = \sum_{j=1}^M P^{\prime}_t(m,j)\Tilde{\mu}^m_{i,j}(h^t_{m,j}) + d_{m,t}\sum_{j \in N_m(t)}\hat{\mu}^m_{i,j}(t) + d_{m,t}\sum_{j \not \in N_m(t)}\hat{\mu}^m_{i,j}(h^t_{m,j})  \notag \\
  & \qquad \qquad \text{ with } d_{m,t} = \frac{1- \sum_{j=1}^M P^{\prime}_t(m,j)}{M} \notag 
\end{align*}

It is worth noting that this global estimator is the weighted average of the local estimators, where the weights can be carefully designed. Precisely, we specify $P_{t}^{\prime}(m,j) = \frac{N - M2^{\frac{1}{\epsilon + 1}}}{MN2^{\frac{1}{\epsilon + 1}}}$ and use mathematical induction to show the following concentration inequality on $\tilde{\mu}_i^m$.  
\begin{lemma}\label{lm:concen_ine}
    Let us assume that Assumption 1 and 2 hold. Let us assume that $p = \frac{\eta}{TM}$ and $P_{t}^{\prime}(m,j) = \frac{N - M2^{\frac{1}{\epsilon + 1}}}{MN2^{\frac{1}{\epsilon + 1}}}$. Let us further assume that $L > 2\kappa (\log{M})^2\log{T}$ where $L$ is the length of the burn-in period. Then for any $m,i$ and $t > L$,  $\Tilde{\mu}_{m,i}(t)$ satisfies that if $n_{m,i}(t) \geq 2(K^2+KM+M)$, then we have the following hold
    \begin{align*}
         & P(\Tilde{\mu}_{m,i}(t) - \mu_i  \geq 2\rho^{ \frac{1}{1+\epsilon} }\bigg(\frac{2Nc \log(t) }{\min_{m}n_{m,i}(t)}\bigg)^{ \frac{\epsilon}{1 + \epsilon} }) \leq \frac{1}{t^2}, \\
         & P(\mu_i - \Tilde{\mu}_{m,i}(t) \geq 2\rho^{ \frac{1}{1+\epsilon} }\bigg(\frac{2Nc \log(t) }{\min_{m}n_{m,i}(t)}\bigg)^{ \frac{\epsilon}{1 + \epsilon} } ) \leq \frac{1}{t^2}.
    \end{align*}
\end{lemma}

\begin{proof}[Proof of Lemma \ref{lm:concen_ine}]

We prove the conclusion through mathematical induction. 

First, at the end of the burn-in period, we have that based on Lemma \ref{lemma: concentration, median of means}, we derive that when $t = L$
\begin{align*}
    & | \tilde{\mu}_i^m(t) - \mu_i | \\
    & = |\hat{\mu}_i^m(t)  - \mu_i | \leq 2C\rho^{ \frac{1}{1+\epsilon} }\bigg(\frac{c \log(t) }{|rw_t|}\bigg)^{ \frac{\epsilon}{1 + \epsilon} } \\
    & 
    \text{with probability at least }1 - \frac{1}{t^2}
    \qquad \forall n \geq 1.
\end{align*}
where $rw_t$ is the total number of samples being used in computing $\tilde{\mu}_i^m(t)$. In our case, we have $rw_t = \frac{L}{K} = n_{m,i}(t) \geq \min_{m,i}n_{m,i}(t)$. Then we derive that with probability at least $1 - \frac{1}{t^2}$, 
\begin{align*}
    & | \tilde{\mu}_i^m(t) - \mu_i | \\
    & \leq 2C\rho^{ \frac{1}{1+\epsilon} }\bigg(\frac{c \log(t) }{\min_{m,i}n_{m,i}(t)}\bigg)^{ \frac{\epsilon}{1 + \epsilon} } \\
    & \leq 2\rho^{ \frac{1}{1+\epsilon} }\bigg(\frac{2Nc \log(t) }{\min_{m}n_{m,i}(t)}\bigg)^{ \frac{\epsilon}{1 + \epsilon} }
\end{align*}
by definition, which proves the statement for $t = L$.

Now let us assume that for any $s \leq t$, we have 
\begin{align*}
    P(\Tilde{\mu}_{m,i}(s) - \mu_i  \geq 2\rho^{ \frac{1}{1+\epsilon} }\bigg(\frac{2Nc \log(s) }{\min_{m}n_{m,i}(t)}\bigg)^{ \frac{\epsilon}{1 + \epsilon} }) \leq \frac{1}{s^2}.
\end{align*} 

At time $t+1$, we have that with $P_{t}^{\prime}(m,j) = \frac{N - M2^{\frac{1}{\epsilon + 1}}}{MN2^{\frac{1}{\epsilon + 1}}}$
\begin{align*}
  & \Tilde{\mu}^m_i(t+1) = \sum_{j=1}^M P^{\prime}_t(m,j)\Tilde{\mu}^m_{i,j}(t_{m,j}) + d_{m,t}\sum_{j \in N_m(t)}\hat{\mu}^m_{i,j}(t) + d_{m,t}\sum_{j \not \in N_m(t)}\hat{\mu}^m_{i,j}(t_{m,j})  \notag \\
  & \qquad \qquad \text{ with } d_{m,t} = \frac{1- \sum_{j=1}^M P^{\prime}_t(m,j)}{M} \notag 
\end{align*}
and thus the difference between $\tilde{\mu}_i^m(t)$ and $\mu_i$ is bounded by 
\begin{align*}
   & |\tilde{\mu}_i^m(t+1) - \mu_i| \\
   & = |\sum_{j=1}^M P^{\prime}_t(m,j)\Tilde{\mu}^m_{i,j}(t_{m,j}) + d_{m,t}\sum_{j \in N_m(t)}\hat{\mu}^m_{i,j}(t) + d_{m,t}\sum_{j \not \in N_m(t)}\hat{\mu}^m_{i,j}(t_{m,j}) - \mu_i| \\
   & = |\sum_{j=1}^M P^{\prime}_t(m,j)(\Tilde{\mu}^m_{i,j}(t_{m,j}) - \mu_i) + d_{m,t}\sum_{j \in N_m(t)}(\hat{\mu}^m_{i,j}(t) -\mu_i) + d_{m,t}\sum_{j \not \in N_m(t)}(\hat{\mu}^m_{i,j}(t_{m,j}) - \mu_i)| \\
   & \leq \sum_{j=1}^{M}P^{\prime}_t(m,j)|\Tilde{\mu}^m_{i,j}(t_{m,j}) - \mu_i| + \\
   & \qquad d_{m,t}\sum_{j \in N_m(t)}|(\hat{\mu}^m_{i,j}(t) -\mu_i)| +  d_{m,t}\sum_{j \not \in N_m(t)}|\hat{\mu}^m_{i,j}(t_{m,j}) - \mu_i| \\
   & \leq \sum_{j=1}^{M}P^{\prime}_t(m,j) \cdot 2\rho^{ \frac{1}{1+\epsilon} }\bigg(\frac{2Nc \log(s) }{\min_{m}n_{m,i}(t_{m,j})}\bigg)^{ \frac{\epsilon}{1 + \epsilon} } + \\
   & \qquad d_{m,t}\sum_{j \in N_m(t)}2C\rho^{ \frac{1}{1+\epsilon} }\bigg(\frac{c \log(t) }{n_{m,j}(t_{m,j})}\bigg)^{ \frac{\epsilon}{1 + \epsilon} } \\
   & \qquad + d_{m,t}\sum_{j \not \in N_m(t)}2C\rho^{ \frac{1}{1+\epsilon} }\bigg(\frac{c \log(t) }{n_{m,j}(t_{m,j})}\bigg)^{ \frac{\epsilon}{1 + \epsilon} } \\
      & \leq \sum_{j=1}^{M}P^{\prime}_t(m,j) \cdot 2\rho^{ \frac{1}{1+\epsilon} }\bigg(\frac{2Nc \log(s) }{\min_{m}n_{m,i}(t)}\bigg)^{ \frac{\epsilon}{1 + \epsilon} } + \\
   & \qquad d_{m,t}\sum_{j \in N_m(t)}2C\rho^{ \frac{1}{1+\epsilon} }\bigg(\frac{2c \log(t) }{\min_{m}n_{m,j}(t)}\bigg)^{ \frac{\epsilon}{1 + \epsilon} } \\
   & \qquad + d_{m,t}\sum_{j \not \in N_m(t)}2C\rho^{ \frac{1}{1+\epsilon} }\bigg(\frac{c \log(t) }{2\min_{m}n_{m,j}(t)}\bigg)^{ \frac{\epsilon}{1 + \epsilon} } \\
   & = \frac{N-M2^{\frac{1}{\epsilon + 1}}}{N2^{\frac{1}{\epsilon + 1}}} 2\rho^{ \frac{1}{1+\epsilon} }\bigg(\frac{2Nc \log(s) }{\min_{m}n_{m,i}(t)}\bigg)^{ \frac{\epsilon}{1 + \epsilon} } \\
   & \qquad + M\frac{1 - \frac{N-M2^{\frac{1}{\epsilon + 1}}}{N2^{\frac{1}{\epsilon + 1}}}}{M}2C\rho^{ \frac{1}{1+\epsilon} }\bigg(\frac{2c \log(t) }{\min_{m}n_{m,j}(t)}\bigg)^{ \frac{\epsilon}{1 + \epsilon} } \\
   & = \frac{N-M2^{\frac{1}{\epsilon + 1}}}{N2^{\frac{1}{\epsilon + 1}}} 2\rho^{ \frac{1}{1+\epsilon} }\bigg(\frac{2Nc \log(s) }{\min_{m}n_{m,i}(t)}\bigg)^{ \frac{\epsilon}{1 + \epsilon} } + \\
   & \qquad (1 - \frac{N-M2^{\frac{1}{\epsilon + 1}}}{N2^{\frac{1}{\epsilon + 1}}}) \cdot 2C\rho^{ \frac{1}{1+\epsilon} }\bigg(\frac{2c \log(t) }{\min_{m}n_{m,j}(t)}\bigg)^{ \frac{\epsilon}{1 + \epsilon} }
\end{align*}
where the second inequality uses the supposition in the mathematical induction, and the concentration inequality for the estimators $\hat{\mu}_i^m$, the third inequality uses Lemma \ref{lemma:information_delay}, and the last inequality uses the definition of $P^{\prime}_t(m,j)$ and $d_{m,t}$.

It is worth noting that when $N > C^{\frac{1+\epsilon}{\epsilon}}$, we have 
\begin{align*}
    & 2\rho^{ \frac{1}{1+\epsilon} }\bigg(\frac{2Nc \log(s) }{\min_{m}n_{m,i}(t)}\bigg)^{ \frac{\epsilon}{1 + \epsilon} } \\
    & \geq 2C\rho^{ \frac{1}{1+\epsilon} }\bigg(\frac{2c \log(t) }{\min_{m}n_{m,j}(t)}\bigg)^{ \frac{\epsilon}{1 + \epsilon} }
\end{align*}

Subsequently, we obtain that
\begin{align*}
    & |\tilde{\mu}_i^m(t+1) - \mu_i| \\
    & \leq \frac{N-M2^{\frac{1}{\epsilon + 1}}}{N2^{\frac{1}{\epsilon + 1}}} 2\rho^{ \frac{1}{1+\epsilon} }\bigg(\frac{2Nc \log(s) }{\min_{m}n_{m,i}(t)}\bigg)^{ \frac{\epsilon}{1 + \epsilon} } + \\
   & \qquad (1 - \frac{N-M2^{\frac{1}{\epsilon + 1}}}{N2^{\frac{1}{\epsilon + 1}}}) \cdot 2C\rho^{ \frac{1}{1+\epsilon} }\bigg(\frac{2c \log(t) }{\min_{m}n_{m,j}(t)}\bigg)^{ \frac{\epsilon}{1 + \epsilon} } \\
   & \leq \frac{N-M2^{\frac{1}{\epsilon + 1}}}{N2^{\frac{1}{\epsilon + 1}}} 2\rho^{ \frac{1}{1+\epsilon} }\bigg(\frac{2Nc \log(s) }{\min_{m}n_{m,i}(t)}\bigg)^{ \frac{\epsilon}{1 + \epsilon} } + \\
   & \qquad (1 - \frac{N-M2^{\frac{1}{\epsilon + 1}}}{N2^{\frac{1}{\epsilon + 1}}}) \cdot 2\rho^{ \frac{1}{1+\epsilon} }\bigg(\frac{2Nc \log(t) }{\min_{m}n_{m,j}(t)}\bigg)^{ \frac{\epsilon}{1 + \epsilon} } \\
   & = 2\rho^{ \frac{1}{1+\epsilon} }\bigg(\frac{2Nc \log(t) }{\min_{m}n_{m,j}(t)}\bigg)^{ \frac{\epsilon}{1 + \epsilon} }
\end{align*}
which subsequently proves the result for $s = t+1$. 

Consequently, we finish the mathematical induction, and conclude the proof of the claim.
\end{proof}

Subsequently, we derive the following upper bound on the number of arm pulls following a different argument compared to the one in the proof of Theorem \ref{thm:1a2}, due to the difference in the arm pulling strategy in the algorithm.

We consider the variant of the UCB strategy used in Algorithm~\ref{alg:dr}, and show that there can only be the following 4 possible scenarios when $a_t^m = i$, which means that arm $i$ is pulled by client $m$:  
\begin{itemize}
    \item Case 1: $n_{m,i}(t) \leq N_{m,i}(t) - 2\kappa \log{M}\log{T}$,\;
    \item Case 2: $\Tilde{\mu}_{m,i} - \mu_i > \rho^{ \frac{1}{1+\epsilon} }\bigg(\frac{2Nc \log(t) }{\min_{m}n_{m,i}(t)}\bigg)^{ \frac{\epsilon}{1 + \epsilon} },$\;
    \item Case 3: $- \Tilde{\mu}_{m,i^*} + \mu_{i^*} > \rho^{ \frac{1}{1+\epsilon} }\bigg(\frac{2Nc \log(t) }{\min_{m}n_{m,i}(t)}\bigg)^{ \frac{\epsilon}{1 + \epsilon} }$,\;
    \item Case 4: $\mu_{i^*} - \mu_i < \rho^{ \frac{1}{1+\epsilon} }\bigg(\frac{2Nc \log(t) }{\min_{m}n_{m,i}(t)}\bigg)^{ \frac{\epsilon}{1 + \epsilon} }$.
\end{itemize}

Translating the scenarios into the value of the number of pulls $n_{m,i}(t)$, we obtain 
\begin{align*}
    n_{m,i}(T) & \leq l + \sum_{t=L+1}^T1_{\{a_t^m = i, n_{m,i}(t) > l\}}  \\
    & \leq l + \sum_{t=L+1}^T1_{\{\Tilde{\mu}^m_i - \rho^{ \frac{1}{1+\epsilon} }\bigg(\frac{2Nc \log(t) }{\min_{m}n_{m,i}(t)}\bigg)^{ \frac{\epsilon}{1 + \epsilon} } > \mu_i, n_{m,i}(t-1) \geq l\}} \\
    & \qquad \qquad + \sum_{t=L+1}^T1_{\{\Tilde{\mu}^m_{i^*} + \rho^{ \frac{1}{1+\epsilon} }\bigg(\frac{2Nc \log(t) }{\min_{m}n_{m,i}(t)}\bigg)^{ \frac{\epsilon}{1 + \epsilon} } < \mu_{i^*}, n_{m,i}(t-1) \geq l\}} \\
    & \qquad \qquad + \sum_{t=L+1}^T1_{\{n_{m,i}(t) < \mathcal{N}_{m,i}(t) - K, a_t^m = i, n_{m,i}(t-1) \geq l\}} \\ 
    & \qquad \qquad + \sum_{t=L+1}^T1_{\{\mu_i + 2\rho^{ \frac{1}{1+\epsilon} }\bigg(\frac{2Nc \log(t) }{\min_{m}n_{m,i}(t)}\bigg)^{ \frac{\epsilon}{1 + \epsilon} }> \mu_{i^*},n_{m,i}(t-1) \geq l\}}. 
\end{align*}

We then take the expectation of $n_{m,i}(t)$ conditional on event $A_{\zeta, \delta}$, and derive 
\begin{align}\label{eq:en}
    & E[n_{m,i}(T) | A_{\zeta, \delta}] \notag \\
    & = l + \sum_{t=L+1}^TP(\Tilde{\mu}^m_i - \rho^{ \frac{1}{1+\epsilon} }\bigg(\frac{2Nc \log(t) }{\min_{m}n_{m,i}(t)}\bigg)^{ \frac{\epsilon}{1 + \epsilon} } > \mu_i, n_{m,i}(t-1) \geq l| A_{\zeta, \delta}) \notag \\
    & \qquad \qquad + \sum_{t=L+1}^TP(\Tilde{\mu}^m_{i^*} + \rho^{ \frac{1}{1+\epsilon} }\bigg(\frac{2Nc \log(t) }{\min_{m}n_{m,i}(t)}\bigg)^{ \frac{\epsilon}{1 + \epsilon} } < \mu_{i^*}, n_{m,i}(t-1) \geq l | A_{\zeta, \delta}) \notag \\
    & \qquad \qquad + \sum_{t=L+1}^TP(n_{m,i}(t) < \mathcal{N}_{m,i}(t) - K, a_t^m = i, n_{m,i}(t-1) \geq l | A_{\zeta, \delta}) \notag \\ 
    & \qquad \qquad + \sum_{t=L+1}^TP(\mu_i + 2\rho^{ \frac{1}{1+\epsilon} }\bigg(\frac{2Nc \log(t) }{\min_{m}n_{m,i}(t)}\bigg)^{ \frac{\epsilon}{1 + \epsilon} } > \mu_{i^*},n_{m,i}(t-1) \geq l | A_{\zeta, \delta}) \notag \\
    & = l + \sum_{t=L+1}^TP(Case 2, n_{m,i}(t-1) \geq l | A_{\zeta, \delta}) + \sum_{t=L+1}^TP(Case 3, n_{m,i}(t-1) \geq l | A_{\zeta, \delta}) \notag \\
    & \qquad  + \sum_{t=L+1}^TP(Case 1, a_t^m = i, n_{m,i}(t-1) \geq l | A_{\zeta, \delta}) + \sum_{t=L+1}^TP(Case 4,n_{m,i}(t-1) \geq l | A_{\zeta, \delta}) 
\end{align}
where $l = \max{\{\frac{2cN\log{T}}{(\frac{\Delta_i}{2\rho^{\frac{1}{1+\epsilon}}})^{\frac{1+\epsilon}{\epsilon}}}, 2\kappa (\log{M})^2 \log{T} \}}$.

We bound the last term in  (\ref{eq:en}) by 
\begin{align}\label{eq:case4}
    \sum_{t= L+1}^TP(Case4: \mu_i + 2\rho^{ \frac{1}{1+\epsilon} }\bigg(\frac{2Nc \log(t) }{\min_{m}n_{m,i}(t)}\bigg)^{ \frac{\epsilon}{1 + \epsilon} } > \mu_{i^*},n_{m,i}(t-1) \geq l) = 0
\end{align}
by the fact that $l \geq \frac{2cN\log{T}}{(\frac{\Delta_i}{2\rho^{\frac{1}{1+\epsilon}}})^{\frac{1+\epsilon}{\epsilon}}}$ with $\Delta_i = \mu_{i^*} - \mu_i$.

Considering the first two terms, we obtain that on $A_{\zeta, \delta}$
\begin{align}\label{eq:cases}
    & \sum_{t= L + 1}^TP(Case 2, n_{m,i}(t-1) \geq l | A_{\zeta, \delta}) + \sum_{t=1}^T P(Case 3, n_{m,i}(t-1) \geq l | A_{\zeta, \delta}) \notag \\
    & \leq \sum_{t= L + 1}^TP(\Tilde{\mu}_{m,i} - \mu_i > \frac{2cN\log{T}}{(\frac{\Delta_i}{2\rho^{\frac{1}{1+\epsilon}}})^{\frac{1+\epsilon}{\epsilon}}} | A_{\zeta, \delta}) \notag \\
    & \qquad + \sum_{t=1}^T P( - \Tilde{\mu}_{m,i^*} + \mu_{i^*} > \frac{2cN\log{T}}{(\frac{\Delta_i}{2\rho^{\frac{1}{1+\epsilon}}})^{\frac{1+\epsilon}{\epsilon}}} | A_{\zeta, \delta}) \notag \\
    & \leq \sum_{t=1}^T(\frac{1}{t^2} ) + \sum_{t=1}^T(\frac{1}{t^2}) \leq \frac{\pi^2}{3}
\end{align}
where the first inequality is true when $n_{m,i}(t-1) \geq l$ and the second inequality results from Lemma~\ref{lm:concen_ine}.

For Case 1, we note that Lemma~\ref{lemma:information_delay} implies that 
\begin{align*}
    n_{m,i}(t) > N_{m,i}(t) - 2\kappa (\log{M})^2 \log{T}
\end{align*}
with the definition of $N_{m,i}(t+1) = \max \{n_{m,i}(t+1),N_{j,i}(t), j \in \mathcal{N}_m(t)\}$.

Based on the observation that the difference between $N_{m,i}(t)$ and $n_{m,i}(t)$ is at most $2\kappa (\log{M})^2 \log{T}$, we next show the exact time steps we need to explore in order to make sure $ - n_{m,i}(t) +N_{m,i}(t)$ to be smaller than $2\kappa (\log{M})^2 \log{T}$. 

At time step $t$, if Case 1 holds for client $m$, then $n_{m,i}(t+1)$ increases by $1$ based on $n_{m,i}(t)$. The following discussion characterizes the change in $N_{m,i}(t+1)$. For client $m$, if $n_{m,i}(t) \leq \mathcal{N}_{m,i}(t) - 2\kappa (\log{M})^2 \log{T}$, the value of $N_{m,i}(t+1)$ remains unchanged, as defined by $N_{m,i}(t+1) = \max \{n_{m,i}(t+1), N_{j,i}(t) : j \in \mathcal{N}_m(t)\}$. 

Additionally, for any client $j \in \mathcal{N}_m(t)$ such that $n_{j,i}(t) < \mathcal{N}_{j,i}(t) - 2\kappa (\log{M})^2 \log{T}$, the value $\mathcal{N}_{j,i}(t+1)$ is not affected since $n_{j,i}(t+1) \leq n_{j,i}(t) + 1$. Consequently, such clients do not influence the value of $N_{m,i}(t+1)$, which remains defined as $N_{m,i}(t+1) = \max \{n_{m,i}(t+1), N_{j,i}(t) : j \in \mathcal{N}_m(t)\}$. Now, Let us consider a client $j \in \mathcal{N}_m(t)$ with $n_{j,i}(t) > \mathcal{N}_{j,i}(t) - 2\kappa (\log{M})^2 \log{T}$. If such a client does not sample arm $i$, the value of $N_{j,i}(t)$ remains unchanged, leading to a decrease of 1 in the difference $-n_{m,i}(t) + N_{m,i}(t)$. On the other hand, if this client samples arm $i$, $N_{m,i}(t)$ increases by 1, keeping the difference between $n_{m,i}(t)$ and $N_{m,i}(t)$ unchanged. However, this scenario falls under Cases 2 and 3, whose total duration has already been upper-bounded by $\frac{\pi^2}{3}$, as shown in (\ref{eq:cases}).

Subsequently, we establish that the time frame of the exploration does not exceed $2\kappa (\log{M})^2 \log{T} + \frac{\pi^2}{3}$, i.e. 
\begin{align}\label{eq:case1}
    & \sum_{t=1}^TP(Case 1, a_t^m = i, n_{m,i}(t-1) \geq l | A)  \notag  \\
    & \leq 2\kappa (\log{M})^2 \log{T} + \frac{\pi^2}{3}.
\end{align}

Consequently, we establish that 
\begin{align*}
    & E[n_{m,i}(T) | A_{\zeta, \delta}] \\ & \leq l + \frac{\pi^2}{3}  + 2\kappa (\log{M})^2 \log{T} + \frac{\pi^2}{3}  + 0 \\
    & =  l + \frac{2\pi^2}{3} + 2\kappa (\log{M})^2 \log{T}  \\
    & = \max{\{\frac{2cN\log{T}}{(\frac{\Delta_i}{2\rho^{\frac{1}{1+\epsilon}}})^{\frac{1+\epsilon}{\epsilon}}}, 2(K^2+MK+M) \}}+  \frac{2\pi^2}{3} + 2\kappa (\log{M})^2 \log{T}
\end{align*}
where the inequality results from (\ref{eq:en}), (\ref{eq:case4}), (\ref{eq:cases}), and (\ref{eq:case1}).

Then again, we consider the aforementioned regret decomposition (which holds by definition and thus does not rely on any reward property), which gives us that
\begin{align*}
   R_T & \leq L+ ((T - L) \cdot \mu_{i^*} - \frac{1}{M}\sum_{m=1}^M\sum_{i = 1}^Kn_{m,i}(T)\mu^m_i) \\
  & \leq 2\kappa (\log{M})^2 \log{T} +  \sum_{m}\sum_{i}n_{m,i}(t)\Delta_i.
\end{align*}

By taking the expected value of $R_T$ given event $A_{\zeta, \delta}$, we derive that 
\begin{align*}
    & E[R_T | A_{\zeta, \delta}] \\
    & \leq  2\kappa (\log{M})^2 \log{T} + \sum_{m}\sum_{i}E[n_{m,i}(t)|A_{\epsilon,\delta}]\Delta_i \\
    & \leq 2\kappa (\log{M})^2 \log{T} + \\
    & \qquad \sum_m\sum_i \Delta_i \cdot (\max{\{\frac{2cN\log{T}}{(\frac{\Delta_i}{2\rho^{\frac{1}{1+\epsilon}}})^{\frac{1+\epsilon}{\epsilon}}}, 2(K^2+MK+M) \}}+  \frac{2\pi^2}{3} + 2\kappa (\log{M})^2 \log{T}) \\
    & \leq 2\kappa (\log{M})^2 \log{T} + \\
    & \qquad \sum_{i}M\Delta_i \cdot (\max{\{\frac{2cN\log{T}}{(\frac{\Delta_i}{2\rho^{\frac{1}{1+\epsilon}}})^{\frac{1+\epsilon}{\epsilon}}}, 2\kappa \log{M}\log{T} \}}+  \frac{2\pi^2}{3} + 2\kappa (\log{M})^2 \log{T}).
\end{align*}

This concludes the proof of Theorem \ref{thm:heter}. 
\end{proof}

\end{document}